\documentclass{article}
\usepackage[utf8]{inputenc}

\usepackage{amsmath, amsthm,amsfonts,amssymb,bbold}
\usepackage{mathrsfs}
\usepackage{tikz-cd}
\usepackage[hyperindex=true,bookmarks=true,bookmarksnumbered=true]{hyperref}
\usepackage[all]{xy}
\newtheorem*{theorem}{Theorem}
\newtheorem*{proposition}{Proposition}
\newtheorem*{lemma}{Lemma}
\newtheorem*{corollary}{Corollary}
\newtheorem*{remark}{Remark}
\newtheorem*{definition}{Definition}
\providecommand{\subjclass}[2][2020]{%
  \par\smallskip\noindent\textit{#1 Mathematics Subject Classification.} #2\par}

\usepackage{algorithm}
\usepackage{algpseudocode}

\usepackage[margin=.3cm]{caption}
\allowdisplaybreaks

\newcommand{\R}{\mathbb{R}}

\def\o{\operatorname{\circ}\,}
\def\X{\mathfrak X}
\def\al{\alpha}
\def\be{\beta}
\def\ga{\gamma}
\def\de{\delta}
\def\ep{\varepsilon}
\def\ze{\zeta}

\def\la{\lambda}

\def\ph{\varphi}

\def\ps{\psi}

\def\Ga{\Gamma}
\def\De{\Delta}

\def\Ph{\Phi}

\def\i{^{-1}}
\def\x{\times}
\def\p{\partial}
\let\on=\operatorname

\usepackage{mathtools}
\DeclarePairedDelimiterX{\norm}[1]{\lVert}{\rVert}{#1}

\DeclareMathOperator*{\argmin}{arg\,min}
\DeclareMathOperator{\Diff}{Diff}

\DeclareMathOperator{\SE}{SE}
\DeclareMathOperator{\SO}{SO}

\DeclareMathOperator{\Id}{Id}
\DeclareMathOperator{\Prirrot}{Pr^\|}
\DeclareMathOperator{\Prsol}{Pr^\perp}
\renewcommand{\SS}{\mathbb S}

\usepackage{xcolor}

\title{Landmark shape spaces with induced metrics}
\author{Sarang Joshi, Peter W. Michor and 
Stefan Sommer}
\date{}

\begin{document}

\maketitle

\begin{abstract}
    We present a unification of Kendall's landmark shape spaces, where rigid motions are factored out and scale fixed on landmark configurations equipped with Euclidean geometry, with landmark configuration spaces carrying Riemannian metrics descending from right-invariant Sobolev metrics on the diffeomorphism group. The resulting new landmark shape spaces achieve the defining properties of both approaches: The regularity of the descending metric prevents landmarks from colliding, the metric is defined in the ambient space independent of the number of landmarks, local rigid transformations are preserved, global rigid motions are removed, and scale fixed. To achieve this, we define a particular Sobolev-type operator, the screened elasticity operator, whose null-space consists exactly of the rigid motions, we show how this operator descends to achieve the desired geometry, and we present approaches to solving matching problems and computing geodesics numerically. The resulting construction allows the use of landmark configuration spaces with sufficiently regular metrics in applications while retaining the shape invariances that are a hallmark of Kendall's shape spaces.
\end{abstract}


\section{Introduction}
Kendall's shape space \cite{kendallShapeManifoldsProcrustean1984} is one of the fundamental constructions in shape analysis representing configurations of landmarks in $\R^d$ modulo translation and rotation, i.e. rigid motions, and fixing scale. Shape analysis based on Kendall's shape spaces is used in a range of applied fields, including morphometry and medical imaging. The central idea is that shapes as an intrinsic feature should be invariant to the arbitrary rotation, translation and scaling when concretely observing or representing the object. A metric structure on landmark shape space is obtained by inducing the Euclidean metric in $(\R^d)^n$ before passing to the shape space by modding out the rigid rotations and translations. Distance between shapes is computed as geodesic distance for the resulting Riemannian metric. In this construction there is no guarantee of initially distinct landmarks not coinciding along the geodesic, and the metric structure is inherently tied to the specific number of landmarks. 

Another approach to shape analysis and landmark shape analysis in particular arises from the pattern theory of Grenander \cite{grenanderGeneralPatternTheory1994} and the modern metric formulation in the Large Deformation Diffeomorphic Metric Mapping (LDDMM) setting \cite{trouveDiffeomorphismsGroupsPattern1998,joshiLandmarkMatchingLarge2000,younesShapesDiffeomorphisms2010}. The core idea is to let diffeomorphism groups of the domain wherein the shape resides act on the shape to produce variations of the shape. By equipping the diffeomorphism group with a suitably invariant Riemannian metric, one obtains induced metrics on the associated landmark shape space. In contrast to Kendall's shape spaces, with suitably regular metrics, initially distinct landmarks can never collide, and any path of changing landmark configurations can be lifted in a well-defined way to a corresponding path of diffeomorphisms removing the dependence on the specific number of landmarks.
\begin{figure}[t]
    \centering
    \includegraphics[width=\textwidth,clip,trim=20 0 400 25]{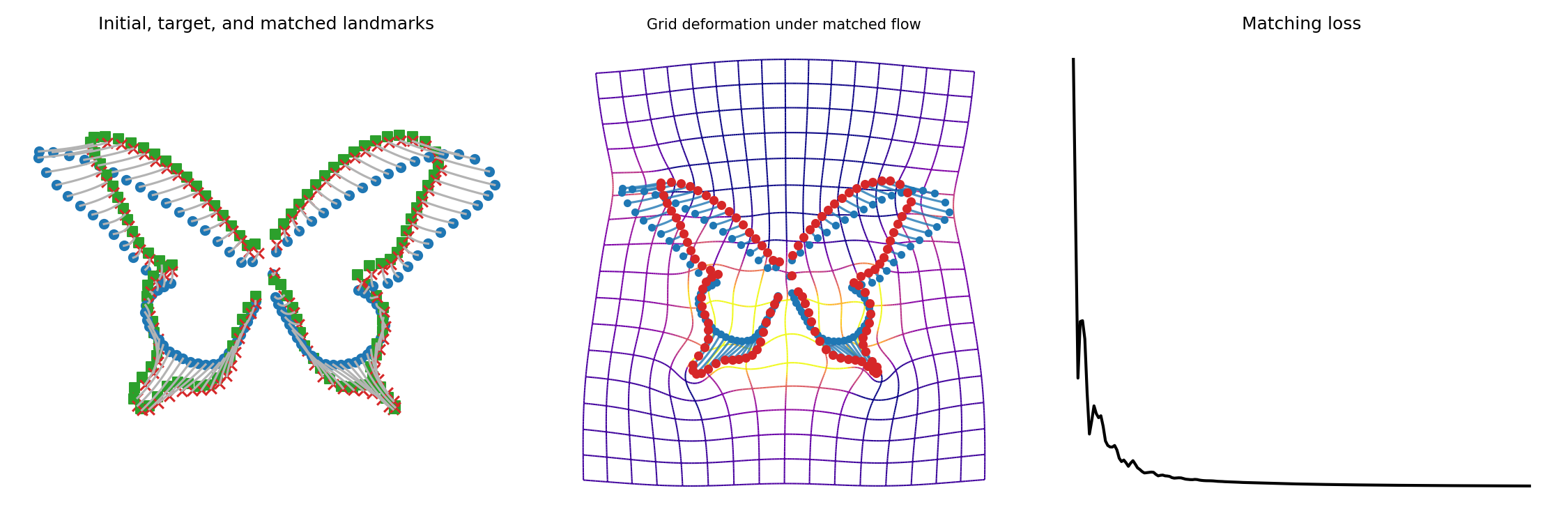}
    \caption{Landmark shape matching of two butterflies with the screened elasticity kernel $k_2$ introduced in this paper. Left: The two butterflies represented by landmarks (blue dots/green squares) and result of the matching (red crosses). Right: Visualization of the deformation induced by the matching visible by the deformation of the initially square grid.}
    \label{fig:butterflies}
\end{figure}

While some works \cite{glaunesTransportParDiffeomorphismes2005,younesCombiningGeodesicInterpolating2006,younesShapesDiffeomorphisms2010} have investigated ways to remove rigid motions when performing LDDMM shape matching, there has been no construction that unifies the two approaches. In this paper, we construct a new shape space that has the defining properties of both approaches: It keeps the regularity of the descending metric to prevent landmarks from colliding and thus preserve shape, it descends from a metric on the diffeomorphism group to avoid tying the metric structure to the number of landmarks, and it factors out rigid motions and fixed scale.

To achieve this, we need an operator that 1) is equivariant with respect to the action of the rigid motion group, 2) has precisely the infinitesimal rigid motions in its null-space, and 3) is sufficiently regular to keep landmarks from colliding. We construct the screened elasticity operator $L_s$ that satisfies these properties, identify its kernel $k_s$, and show how the resulting metric descends to the landmark shape space with rigid motions factored out and scale fixed to achieve the desired properties.
Figure~\ref{fig:butterflies} illustrates the matching of two butterfly shapes with the new landmark shape space.

While the focus of this paper is landmark shape spaces, when extended from landmarks to surfaces and images, this framework provides a single geometric model in which global rigid motion is quotiented out while the remaining non-rigid variability is represented by diffeomorphic flows. Thus, rigid registration and deformable registration are no longer separate preprocessing and matching steps, but parts of one unified metric framework that preserves local rigid motion while measuring true non-rigid shape change.

\subsection{Landmark shape matching}
To motivate the problem, consider a landmark configuration $q=(q_1,\ldots,q_n)$, an ordered set of $n$ points in $\R^d$, and another configuration of corresponding landmarks $q'=(q'_1,\ldots,q'_n)$. A standard problem in shape analysis is to find a transformation that matches these points. If we restrict the allowable transformations to rigid motion and use a least squares matching criterion, the problem is given as
\begin{equation}
    \argmin_{R\in\SO(d),b\in\R^d} \sum_{i=1}^{n}||Rq_i+b - q'_i||^2 
    .
\end{equation}
After solving for the rigid transformation, the landmarks are generally not perfectly matched, and we may need to find a more general diffeomorphic transformation that performs the remaining transformation for the matching. As the rigid transformation is generally not considered part of the shape itself, we would like to have a geometric picture in which the rigid registration is free but there is a metric on the remaining, non-rigid part. Formulating the matching problem in this way leads naturally to consider semi-degenerate metrics on the diffeomorphism group that have the rigid transformations in the null-space.
Once we do the rigid transformation of the landmarks, we do not want to introduce transformations in the remainder that has a rigid part. These concepts are linked via conditionally positive definite (CPD) kernels. The final construction needs a particular operator and connected kernel that gives a Riemannian metric on the resulting quotient space. The corresponding kernel is CPD exactly with respect to the space of infinitesimal rigid motions. The effect of the new kernel when matching shapes differing by local rigid motions is illustrated in Figure~\ref{fig:matching-grid-kernel-comparison}.
\begin{figure}[t]
    \centering
    \includegraphics[width=\textwidth,clip,trim=0 0 0 25]{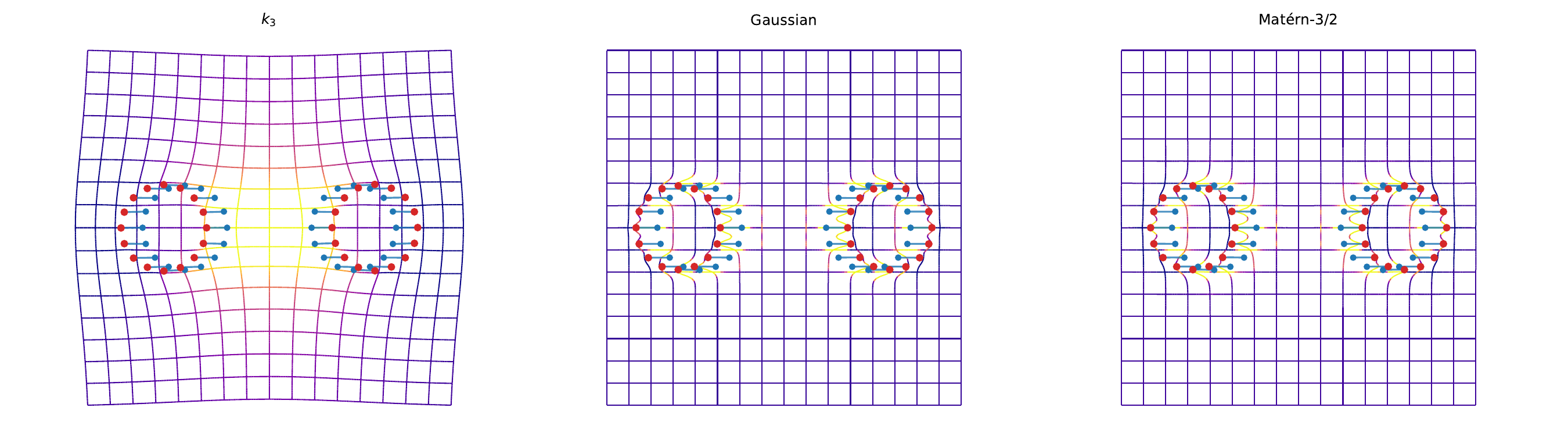}
    \caption{Landmark shape matching example, comparing the screened elasticity kernel $k_2$ (left), the Gaussian kernel (center), and the Mat\'ern kernel $k_{H^{3/2}}$ (right). The initial configuration of two circles (blue dots) differs from the target (red dots) only in local outwards translations in the horizontal direction. The landmark trajectories during the matching are shown together with the deformation induced by the matching visible in the colored grid. The $k_2$ screened elasticity kernel introduced in this paper preserves the local rigid motion inside the circles, while both the Gaussian and Mat\'ern kernel penalize any movement and thus do not move the centers of the circles in a local rigid way: With a rather small value of the kernel width $\sigma$ as in this case, the centers stay almost constant for the Gaussian and Mat\'ern kernel. The preservation of local rigid deformation achieved by the $k_2$ kernel is a main target of the paper.}
    \label{fig:matching-grid-kernel-comparison}
\end{figure}

\subsection{Contributions}
We present a new landmark shape space that has the defining properties of both Kendall's shape spaces and LDDMM shape spaces. From the Kendall shape space point of view, the model has the regularity to prevent points from colliding and the metric is defined in the ambient space thus inducing compatible metrics with varying numbers of landmarks. From the LDDMM viewpoint, the new operator induces a quadratic form with exactly infinitesimal rigid motions in the null-space that descends to the quotients modulo the rigid motion group and preserves local rigid transformations.

To enable the above, we add regularity to the classical elasticity operator to retrieve an operator that has exactly rigid transformations in the null-space, thereby placing itself between standard Sobolev operators and the degenerate Laplacian that leads to Beppo-Levi spaces. We thus get a particular instance of conditionally positive definite kernels, but with conditions on a smaller set of polynomials compared to the Beppo-Levi case. This is important because we cannot factor out larger sets of transformations than the rigid motions as this would break the equivariance of the operator.

\subsection{Applied impact in shape analysis}
Kendall's shape spaces are widely used in geometric morphometrics. However, the Euclidean metric on landmarks implies no correlation between the points on the metric side making it hard to parametrically model covariance between evolving landmarks, and the metric structure is intrinsically dependent on the number of landmarks. The wider family of metrics that we allow solves these issues. Conversely, while LDDMM is used in range of applications including medical imaging and biology, the rigid motion invariance is generally not accounted for in the LDDMM setting, with only an initial preregistration being used as an ad-hoc substitute. The new metrics will directly improve the modeling in this setting. The construction additionally will allow flow-based models to be used, for example stochastic processes based on Kunita flows that are used for modelling change of shape over time in evolutionary biology \cite{sommerStochasticsShapesKunita2026}.

\subsection{Overview and paper outline}
In section~\ref{sec:background}, we describe the necessary background for the paper. Section~\ref{sec:quotient_shape_space} describes the quotient shape space arising from removing rigid transformations. In section~\ref{sec:operator}, we discuss the precise properties needed for the operator and define the screened elasticity operator that precisely has infinitesimal rigid transformations in its null-space and preserves local rigid transformations, interpret its Green's kernel as a conditionally positive definite kernel, and describe the induced geometry on the quotient shape space. We continue in section~\ref{sec:scale} by restricting to fixed-scale submanifolds of the shape space. We conclude the paper by numerical considerations and experiments in section~\ref{sec:implementation} and section~\ref{sec:experiments}. The appendices contain further background material and derivations.

\section{Background}
\label{sec:background}
We here describe the necessary background material on Kendall's and LDDMM landmark shape spaces, the elasticity operator, rigid motions and projections onto the infinitesimal rigid motions. We shall repeatedly use the Euclidean vector space $\mathbb R^d$ with norm $|x|$ and inner product $\langle x,y\rangle$. Other inner products will be distinguished by subscripts like $\langle f,g\rangle_{L^2}$.

\subsection{Kendall's and LDDMM landmark spaces}
\label{sec:kendall}
Let $q=(q_1,\ldots,q_n)$ denote a configuration of $n$ landmarks in $\R^d$. Landmarks are just points, but the notion landmark is generally used when they represent a shape. 
Kendall's shape spaces \cite{kendallDiffusionShape1977} $\Sigma_d^n$ are quotients of spaces of such configurations of $n$ landmarks in $\R^d$ with translation removed, normalized for scale, and taken modulo rotation, i.e., $M/\SO(d)$ where $M$ is the intersection of the set of landmark configurations with zero centroid and the unit sphere $S^{nd}$ in $(\R^d)^n$ (scaling to $\sum_i |q_i|^2 =1$), and where the special orthogonal group $\SO(d)$ acts on landmark configurations diagonally, i.e., by moving each landmark in the configuration. The space $M$ of landmark configurations centered and normalized for scale, but before factoring out rotations, is denoted the pre-shape space. The pre-shape space forms a sphere that inherits the standard Euclidean metric on $(\R^d)^n$. The standard Riemannian metric on the pre-shape space descends under the quotient operation to give a metric on the quotient shape space \cite{leRiemannianStructureEuclidean1993}.

In contrast to Kendall's shape space, the LDDMM landmark spaces $\on{Land}^n = \on{Land}^n(\mathbb R^d)$ require distinct points, i.e. $q_i\neq q_j$ for $i\neq j$, and $\on{Land}^n$ is therefore an open subset of $(\R^d)^n$. While Kendall's shape metric on the pre-shape space is of $L^2$-type with no coupling between the landmarks, the metrics on $\on{Land}^n$ are induced from metrics on $\Diff_{\mathcal A}(\R^d)$, typically Sobolev metrics.  The construction is described in Appendix~\ref{sec:induced} and the specific diffeomorphism groups $\Diff_{\mathcal{A}}(\mathbb R^d)$ used in the paper and their Lie algebras $\X_{\mathcal{A}}(\mathbb R^d)$ are given in Appendix~\ref{sec:groups}. 

 The construction applies to shape spaces beyond landmarks, for example closed curves $c:\SS^1\to \R^d$ or surfaces $s:\SS^2\to \R^d$. While we focus here on the landmark case, the constructions of this paper can be extended more generally to other shape spaces.

Other shape models based on elasticity has been introduced in several works in the shape literature, e.g., the shape averaging and covariance analysis in \cite{doi:10.1137/080738337,doi:10.1137/S0036139902419528}, elasticity based shape metrics \cite{fuchsShapeMetricsBased2009}, and in image registration \cite{doi:10.1137/080738337}. These differ from LDDMM in modeling the displacement between shapes compared to the time integration of tangent vector fields, or in being inner metrics modeling the elastic deformation of the shape itself instead of the embedding space as done in LDDMM.

\subsection{Landmark shape matching}
\label{sec:shape_matching}
Given shapes $q$ and $q'$, shape matching is the problem of finding a diffeomorphism $\phi$ that maps $q$ to the target $q'$ such that the induced metric on the landmark space is minimized. This is equivalent to finding a geodesic in the landmark space that connects $q$ and $q'$. Given an initial configuration $q$ and a target configuration $q'$, this can be phrased as finding an initial momentum $p$ minimizing
\begin{equation*}
    p \mapsto \mathcal L\bigl(q(T;q,p),\, q'\bigr),
\end{equation*}
where $q(T;q,p)$ is the endpoint of the geodesic at time $T$ starting at $q$ with initial momentum $p$, and $\mathcal L$ is a measure of the discrepancy between $q(T;q,p)$ and $q'$, e.g. the sum of squared distances between corresponding landmarks.
This approach to shape matching is also denoted \emph{shooting}.

Landmark matching that disregards rigid or affine motions in the context of LDDMM was described in \cite{glaunesTransportParDiffeomorphismes2005}, where derivatives of rigid motions were included either in the vector field governing the flow or rigidly aligning the final diffeomorphism, and equivalence between the two approaches is shown; in \cite{younesCombiningGeodesicInterpolating2006} where affine alignment of the flow end-state to the target is imposed as a boundary condition; and furthermore in \cite{younesShapesDiffeomorphisms2010} where asymptotically affine kernels are covered.

\subsection{Infinitesimal rigid motions and the elasticity operator}
\label{sec:elasticity_operator}
Since we aim to quotient out rigid motions from the landmark space, we will consider operators whose null-space consists exactly of the infinitesimal rigid motions. We here describe the rigid motion group and the classical elasticity operator.

Let $\SE(d)$ be the special Euclidean group of rigid motions of $\mathbb{R}^d$.
A rigid motion $g\in\SE(d)$ of $\mathbb{R}^d$ has the form $g(x)=Rx+b$ with $R\in\SO(d)$ and $b\in\mathbb{R}^d$.
Let $t\mapsto g(t)=(R(t),b(t))$ be a $C^1$ curve in $\SE(d)$ and set
\[
v(x):=\left.\frac{d}{dt}\right|_{t=0}(R(t)x+b(t))
      =\Omega x+a,
\qquad \Omega:=\dot R(0),\ a:=\dot b(0).
\]
Differentiating $R(t)^{\mathsf T}R(t)=\Id$ at $t=0$ yields $\Omega^{\mathsf T}+\Omega=0$, i.e. $\Omega\in\mathfrak{so}(d)$ is a skew-symmetric matrix.
Hence the Lie algebra of infinitesimal rigid motions is
\[
\mathfrak{se}(d)
:=\{\,x\mapsto \Omega x+a \ :\ a\in\mathbb{R}^d,\ \Omega\in\mathfrak{so}(d)\,\},
\]
a subset of $\R^d$-valued first order polynomials with dimension $\dim(\mathfrak{se}(d))=d+\tfrac{d(d-1)}2=\tfrac{d(d+1)}2$.

The classical elasticity operator on vector fields on $\mathbb R^d$, also known as the Killing operator, $\on{sym}\nabla: C^{\infty}(\mathbb R^d,\mathbb R^d)\to C^{\infty}(\mathbb R^d, S^2\mathbb R^d)$ is given by 
\[
\on{sym}\nabla v = (\partial_i v^j + \partial_j v^i)_{i,j=1}^d
.
\]
It has the infinitesimal rigid motions as its null-space: If $\partial_i v^j + \partial_j v^i=0$ for all $i,j$, then for any $l$ we have 
$0=\p_l\p_i v^j + \p_l\p_j v^i =\p_i\p_lv^j + \p_j\p_lv^i = -\p_i\p_jv^l - \p_j\p_iv^l = -2\p_i\p_jv^l$.
Thus $v(x)= Ax+ b = (A^i_j x^j +b^i)$. Reinserting we get  $0=\p_i v^j + \p_j v^i=A^j_i +A^i_j$ so that $A\in\mathfrak{so}(d)$.

We can identify the corresponding inertia operator by partial integration
\begin{align*}
\int_{\mathbb R^d} \on{Tr}\big(\on{sym}\nabla v.\on{sym}\nabla w \big)dx 
&= \int_{\mathbb R^d} \sum_{i,j=1}^d(\p_i v^j+\p_j v^i)(\p_i w^j + \p_j w^i)dx
\\&
= \int_{\mathbb R^d}\sum_{i,j=1}^d\big((-2\p_i^2 v^j -2 \p_j\p_i v^i)w^j \big)dx
\\&
= \int_{\mathbb R^d} \langle (-2\De - 2\nabla\on{div})v, w\rangle dx
.
\end{align*}
The Fourier transform of the elasticity operator is $\widehat{\on{sym}\nabla v}(\xi) = i\big(\xi\otimes \hat v(\xi) + (\xi\otimes \hat v(\xi) )^T \big)$. From this, define the operator $L_{\on{sym}}=-\on{div}\o \on{sym}\nabla$ with Fourier transform
\begin{equation*}
\hat L_{\on{sym}}\hat v(\xi) = 2\big(|\xi|^2 \hat v(\xi) + \langle \xi,\hat v(\xi)\rangle\xi\big) = 2\big(|\xi|^2\on{Id}  +  \xi\otimes \xi\big)\hat v(\xi)
.
\end{equation*}
The operator is elliptic on the orthogonal complement of its polynomial null-space which coincides exactly with the Lie algebra $\mathfrak{se}(d)$ of infinitesimal rigid motions. Up to the factor 2, $L_{\on{sym}}$ is the classical Helmholtz operator of linear elasticity, or, more generally, the linear elasticity operator
\begin{equation}
L_{\on{Lam\acute{e}}}v = -\mu\De v - (\mu+\lambda)\nabla\on{div}v)
\label{eq:L_Lame}
\end{equation}
with $L_{\on{sym}}$ corresponding to Lamé parameters $\mu=2$ and $\lambda=0$, see e.g. \cite{marsdenMathematicalFoundationsElasticity1994,MathematicalTheoryElastic2017,mcleanStronglyEllipticSystems2000}.

The operator $L_{\on{Lam\acute{e}}}$ has Green's function on the complement of infinitesimal rigid motions in the form of the Kelvin matrix $\Phi$ which, for $d=2$, is given by
\begin{equation*}
    \Phi(x) = \frac{1}{4\pi\mu(2\mu+\lambda)}\left((3\mu+\lambda)\log\frac{1}{|x|}\,I_2+(\mu+\lambda)\frac{xx^T}{|x|^2}\right)
\end{equation*}
and, for $d=3$,
\begin{equation*}
    \Phi(x) = \frac{1}{8\pi\mu(2\mu+\lambda)}\left((3\mu+\lambda)\frac{1}{|x|}I_3+(\mu+\lambda)\frac{xx^T}{|x|^3}\right)
    \quad\quad
\end{equation*}
We will use the elasticity operator $L_{\on{sym}}$, it's kernel, and screened versions of both in the following.

\subsection{Lagrange basis and projector for infinitesimal rigid motions}
\label{sec:lagrange_basis}
To interpret the Kelvin matrix above and later on kernels for the screened elasticity operator, we now specialize the construction of conditionally positive definite kernels (CPD) to the case where the condition space is the Lie 
algebra $\mathfrak{se}(d)$ of infinitesimal rigid motions. While the general CPD construction is outlined in Appendix~\ref{sec:cpd_kernels}, the Kelvin matrix and its screened versions are matrix-valued rather than scalar, so the standard construction need to be adapted to vector-valued point evaluations.

Define linear functionals $\{\xi_i\},\{\xi_{ij}\}$ on vector fields $v\in C^0(\mathbb{R}^d,\mathbb{R}^d)$ by
\[
\xi_i(v):=e_i\cdot v(0)=v_i(0),\qquad i=1,\dots,d,
\]
and, for $1\le i<j\le d$,
\[
\xi_{ij}(v):=e_i\cdot\bigl(v(e_j)-v(0)\bigr)=v_i(e_j)-v_i(0).
\]
For $v(x)=\Omega x+a\in\mathfrak{se}(d)$ one has
\[
\xi_i(v)=a_i,
\qquad
\xi_{ij}(v)=e_i\cdot(\Omega e_j)=\Omega_{ij}\quad (i<j).
\]
In particular, $\{\xi_i\},\{\xi_{ij}\}_{i<j}$ extracts exactly the coefficients $(\Omega,a)$. If $v(x)=\Omega x+a\in\mathfrak{se}(d)$ satisfies
$\xi_i(v)=0$ for $i=1,\ldots,d$ and $\xi_{ij}(v)=0$ for $i<j$, then $a_i=\xi_i(v)=0$ so $a=0$, and $\Omega_{ij}=\xi_{ij}(v)=0$ for all $i<j$.
Since $\Omega$ is skew-symmetric, this implies $\Omega=0$, hence $v=0$.
Therefore the functionals are unisolvent on $\mathfrak{se}(d)$.

Define now basis functions in $\mathfrak{se}(d)$ by
$p_i(x)=e_i$ for $i=1,\ldots,d$ and $p_{ij}(x)=x_j e_i - x_i e_j$ for $1\le i<j\le d$.
Then
$\xi_k(p_i)=\delta_{ki}$ and $\xi_{pq}(p_i)=0$, and $\xi_k(p_{ij})=0$ and $\xi_{pq}(p_{ij})=\delta_{ip}\delta_{jq}$. Therefore $\{p_i\}_{i=1}^d,\{p_{ij}\}_{1\le i<j\le d}$ is a Lagrange basis for $\mathfrak{se}(d)$ with respect to $\{\xi_i\},\{\xi_{ij}\}_{i<j}$.

Define the map $\Pi_{\mathfrak{se}(d)}$ on vector fields by
\[
\Pi_{\mathfrak{se}(d)}(v)(x)
:=
\sum_{1\le i<j\le d}\xi_{ij}(v)\,p_{ij}(x)
+\sum_{i=1}^d \xi_i(v)\,p_i(x).
\]
Then $\Pi_{\mathfrak{se}(d)}(v)\in\mathfrak{se}(d)$, $\xi_i(\Pi_{\mathfrak{se}(d)}(v))=\xi_i(v)$, and $\xi_{ij}(\Pi_{\mathfrak{se}(d)}(v))=\xi_{ij}(v)$. $\Pi_{\mathfrak{se}(d)}$ is a projection onto $\mathfrak{se}(d)$, because, for $v\in\mathfrak{se}(d)$ with $v(x)=\Omega x+a$,
\[
\Pi_{\mathfrak{se}(d)}(v)(x)
=\sum_{1\le i<j\le d}\Omega_{ij}\,p_{ij}(x)+\sum_{i=1}^d a_i\,p_i(x)
=v(x).
\]
For ease of notation, in the following, we let $\lambda_1,\ldots,\lambda_r$ and $p_1,\ldots,p_r$ denote the functionals $\xi_i,\xi_{ij}$ and basis functions $p_i,p_{ij}$, respectively, with $r=d(d+1)/2$.

Let $\Phi$ be a matrix-valued Green's function such as the Kelvin matrix above. For linear functionals of the form $\lambda=\sum_{i=1}^N a_i\otimes\delta_{x_i}$ with $x_i,a_i\in\mathbb R^d$ and $(a\otimes\delta_x)(f)=a^T f(x)$, define
$G_\lambda = \sum_{i=1}^N \Phi(\cdot,x_i)a_i$.
Define $\tilde{\mathcal F}$ as
\begin{equation*}
\tilde{\mathcal F}=
\left\{G_\lambda \,\Big|\, \lambda=\sum_{i=1}^Na_i\otimes\delta_{x_i},\ \lambda(p)=0\,\forall p\in \mathfrak{se}(d)
\right\}
\end{equation*}
and $\mathcal F$ as its completion in the kernel norm $\| G_{\la}\|^2_{\Ph} = \sum_{i,j}\langle a_i,\Ph(x_i-x_j)a_j\rangle$; see \ref{sec:cpd_kernels}. Let now $G_{(\lambda)}$ be the $\mathfrak{se}(d)$-corrected version of $G_\lambda$ defined by
\[
G_{(\lambda)}
= G_\lambda - \sum_{i=1}^r \lambda(p_i)G_{\lambda_i}.
\]
Then $G_{(\lambda)}\in\tilde{\mathcal F}$ because the $r+1$ linear functionals $
\lambda,-\lambda(p_1)\lambda_1,\ldots,-\lambda(p_r)\lambda_r$ satisfy $\lambda(p)-\sum_{i=1}^r \lambda(p_i)\lambda_i(p)=0$ for all $p\in\mathfrak{se}(d)$.
Thus $G_{(\lambda)}$ corresponds to the corrected linear functional
\[
(\lambda)=\lambda-\sum_{i=1}^r \lambda(p_i)\lambda_i,
\]
where $(\lambda)(p)=0$ for all $p\in\mathfrak{se}(d)$. Because $(\lambda)$ vanishes on $\mathfrak{se}(d)$, it has a continuous extension to $\mathcal F$ and we have the reproducing property
\[
(\lambda)(f)
= \langle f,\ G_{(\lambda)}\rangle_{\mathcal F}.
\]
Let $R:\mathcal F\to \mathcal F^{**}$ be the Riesz map $R(f)(\lambda)=\langle f,G_{(\lambda)}\rangle_{\mathcal F}$. Then $R(f)(a\otimes\delta_x)=a^T \left(f(x)-\Pi_{\mathfrak{se}(d)}(f)(x)\right)$ for $f\in\tilde{\mathcal F}$. 
The native space of $\Phi$ is then
\[
\mathcal N_\Phi(\Omega):=R(\mathcal F(\Omega))+\mathfrak{se}(d) 
\]
with the semi-inner product
\[
(f,h)_{\mathcal N_\Phi(\Omega)}
:= \bigl(R^{-1}(f-\Pi_{\mathfrak{se}(d)}f),\ R^{-1}(h-\Pi_{\mathfrak{se}(d)}h)\bigr)_{\mathcal H}.
\]
In section \ref{sec:operator}, we will use this construction with $\Phi$ the Green's function of the screened elasticity operator.

\section{Rigid-motion quotients}
\label{sec:quotient_shape_space}

We here outline how the rigid motion group $\SE(d)$ acts on $\Diff_{\mathcal{A}}(\mathbb R^d)$ and its diagonal action on $\on{Land}^n(\mathbb R^d)$ leading to the quotient shape space. We describe the stratified structure of the shape space, and how the action is mapped to the quotient.

\subsection{The diffeomorphism group modulo the outer automorphism group of rigid motions}
\label{sec:outer_action}
The rigid motion group $\SE(d)$ acts as a group of outer automorphisms on any of the diffeomorphism groups $\Diff_{\mathcal{A}}(\mathbb R^d)$ by conjugation: $\ph\mapsto A\o \ph \o A\i$ for $A\in \SE(d)$. The induced action on the Lie algebra is $X \mapsto R\o X \o A\i$ for  $A(x)=Rx +b$.

\subsection{Landmark space modulo rigid motions}
\label{sec:modSE}
The rigid motion group $\SE(d)$ acts diagonally on landmark space $\on{Land}^n$. This is a proper action. It is polar for the Euclidean metric for $d=1$ and any $n$ or $d\ge2$ and $n=2$, but is not polar for $d\ge2$ and $n\ge3$. For $q\in \on{Land}^n$, the infinitesimal action is the linear map
\begin{gather*}
\zeta_q:\mathfrak{se}(d)\to T_q\on{Land}^n,
\\ 
\zeta_q(\Omega,a):=\frac{d}{dt}\Big|_{t=0}\exp\bigl(t(\Omega,a)\bigr)\cdot q
=(\Omega q_1+a,\ldots,\Omega q_n+a)
\end{gather*}
where $(\Omega,a)\in \mathfrak{se}(d)=\mathfrak{so}(d)\ltimes\mathbb R^d$. We denote the tangent space to the rigid-motion orbit through $q$ by
\[
\mathfrak{se}(d)(q):=\zeta_q\bigl(\mathfrak{se}(d)\bigr)
=
\bigl\{(\Omega q_i+a)_{i=1}^n:(\Omega,a)\in \mathfrak{se}(d)\bigr\}
\subset T_q\on{Land}^n.
\]
and write
\[
\mathfrak{se}(d)(q)^{\mathrm{ann}}
:=
\Bigl\{\al\in T_q^*\on{Land}^n\Big|
\sum_{i=1}^n \al_i(\Omega q_i+a)=0
\;\forall\;(\Omega,a)\in\mathfrak{se}(d)\Bigr\}
\]
for the annihilator of the orbit tangent space, equivalently
$\sum_{i=1}^n \al_i=0$ and $\sum_{i=1}^n (q_i \al_i^T-\al_i q_i^T)=0$.

The action of the translation subgroup is free, and the quotient by all translations is the space of centered landmarks
\begin{equation*}
\on{Land}^n_0:=\on{Land}^n/\mathbb R^d = \{q\in \on{Land}^n: C(q):=\frac1n\sum_i q_i =0\}
\end{equation*}
where $\SO(d)$ is still acting. 
Note that $\on{Land}^n_0$ is not invariant under the action of $\Diff_{\mathcal{A}}(\mathbb R^d)$.

Consider the center mapping, the span mapping, and the rank of a landmark
\begin{align*}
&\bar q = (q_i-C(q))_{i=1}^n = t_{C(q)}^{-1}(q),\quad 
\text{where } t_b \text{ is translation by }b,
\\
&\on{Span}:\on{Land}^n \to \bigcup_{k=1}^d \on{Gr}_n(k),\quad \on{rk}(q) :=\dim(\on{Span}(\bar q)) \in \{1,\dots,n-1\},
\end{align*}
where $\on{Span}$ maps a landmark $q$ to the linear subspace of $\mathbb R^d$ spanned by its centered translate 
$\bar q$.
For a landmark $q$, since $\bar q$ contains a basis of $\on{Span}(\bar q)$, the stabilizer or isotropy group $\SE(d)_q\subset \SE(d)$ is 
\begin{align*}
    \SE(d)_q &= t_{C(q)}.(\SO(\on{Span}(\bar q)^\bot)\x \{\on{Id}_{\on{Span}(\bar q)}\}.t_{C(q)}^{-1}\,,
\end{align*}
which is conjugate to $\SO(d-k)\x \{\on{Id}_k\}$ where $k=\on{rk}(q)$. Thus $\dim(\SE(d)_q) = \dim(\SO(d-k)) = \frac{(d-k)(d-k-1)}{2}$. The conjugacy class of the isotropy subgroup $\SE(d)_q\subset\SE(d)$, also called the orbit type of $q$, depends only on $k$. The orbit space $\on{Land}^n/\SE(d)$ is the following union of strata
\begin{equation}
\overline{\on{Land}}^n := \on{Land}^n(\mathbb R^d)/\SE(d) = \on{Land}^n_0/\on{SO}(d) = \bigcup_{k=1}^{\min\{d,n-1\}} \mathcal{S}_k  
\end{equation}
with $\dim(S_k) = nk -\tfrac12 k(k+1)$. The stratum $\mathcal{S}_k$ is regular (the conjugacy class of the isotropy group is minimal) when $n\ge d$ if $k=d$ or $k=d-1$ since both $\SO(0)$ and $\SO(1)$ are trivial; when $n<d$ if $k=n-1$. The regular stratum is connected, open and dense, see e.g. \cite[6.8 ff]{michorTopicsDifferentialGeometry2008}. Moreover, $\mathcal{S}_{k'}$ is contained in the boundary of $\mathcal{S}_k$ whenever $k'< k$. We denote the projection $\pi:\on{Land}^n\to \overline{\on{Land}}$.
\begin{lemma}
For each $q$, the dual mapping of $T_q\pi: T_q\on{Land}^n \to T_{\pi(q)}\overline{\on{Land}}^n$ factors to an isomorphism 
\[
(T_q\pi)^*: T^*_{\pi(q)}\on{Land}^n/\on{SE}(d)=T^*_{\bar q}\overline{\on{Land}}^n \xrightarrow{\cong} \mathfrak{se}(d)(q)^{\mathrm{ann}} \subset T^*_q\on{Land}^n
\]
\end{lemma}
\begin{proof}
This follows from $T_q(\on{SE}(d).q) = \mathfrak{se}(d)(q)$.    
\end{proof}

\section{Screened elasticity and quotient metrics}
\label{sec:operator}
To achieve a geometric structure on landmark shape spaces that has the defining properties of both Kendall's shape spaces and LDDMM landmark spaces, we need an operator and connected right-invariant metric on $\Diff_{\mathcal A}(\mathbb R^d)$ that has specific properties. We discuss these properties below before defining the screened elasticity operator.

\subsection{Operators and descending metrics}
\emph{Equivariance}: To define a metric on the quotient space, the operator must be equivariant with respect to the action of the group we factor out.
The rigid motion group $\SE(d)$ is the connected component of the isometry group of $\mathbb R^d$ for the standard Euclidean Riemannian metric. Therefore, every operator which uses only operations derived from the Euclidean metric like the Laplacian, divergence and curl are equivariant with respect to the rigid motion group. 

Several $\SE(d)$-equivariant operators can be used to induce quadratic forms on quotient shape spaces: the Sobolev operators $L_{H^s}$, the Laplacian operators $L_{\Delta^s}$ (see Appendix~\ref{sec:operator}), and the screened elasticity operator $L_s$ introduced below. All of them descend to the quotient, but they induce quadratic forms with different properties. The choice of rigid motion group to factor out is important: The operators are not equivariant with respect to the larger group of affine or similarity transformations. In Appendix~\ref{app:affine-equivariant-operators} we prove that the only linear differential operator that is invariant under affine transformations is of order 0, thus a multiple of the identity.

\emph{Null-space}: The null-space of the operator must contain exactly the infinitesimal rigid motions as we wish to preserve local rigid transformations and induce a metric on the quotient. A smaller null-space will not preserve local rigid transformations as illustrated in Figure~\ref{fig:matching-grid-kernel-comparison}, while a larger null-space will lead to a degenerate quadratic form on the quotient and hence not define a metric. A degenerate quadratic form would imply that shapes differing by e.g. a shearing transformation would have distance zero. In addition, since the operator is not equivariant with respect to the affine group, the distance between a pair of shapes would not be the same as the distance between the pair of shapes after a shearing transformation, even though the shearing transformation induced zero distance.

For the Sobolev operators $L_{H^s}$, the identity term in the operators implies that the induced metric will penalize every non-zero component of the vector field. After factoring out rigid motions, global rigid transformations are removed, but local rigid motions are still penalized.
For the Laplacian operators $L_{\Delta^s}$, the null-space include all affine transformations. Because the null-space is larger than the group factored out, the induced quadratic form is degenerate and does not define a metric.
For the elasticity operator $L_{\on{sym}}$, the null-space is exactly $\mathfrak{se}(d)$. 

\emph{Regularity}: The kernel of the operator must be sufficiently regular to induce a metric that prevents landmarks from colliding.
The kernel of the elasticity operator $L_{\on{sym}}$ has a singularity at $0$ and because of this low regularity, we cannot expect geodesic completeness. The screened elasticity operator $L_s$ is constructed to have the necessary regularity by increasing the order of $L_{\on{sym}}$.

In summary, the Sobolev operators has a too small null-space and thus penalize local rigid motions; the Laplacian has a too large null-space failing to induce a metric; and the elasticity operator has too low order. The screened elasticity operator $L_s$ has all of the desired properties. We now construct $L_s$ below.

\subsection{Screened elasticity}
From the elasticity operator $L_{\on{sym}}$ described in section \ref{sec:elasticity_operator}, we obtain the screened elasticity operator $L_s$ as
\begin{equation*}
     L_s=-(\on{Id} -\sigma^2\De)^{s-1}\o\on{div}\o \on{sym}\nabla   
\end{equation*}
for $s\ge 1$. The elasticity operator $L_{\on{sym}}$ that has order 2 corresponds to $s=1$, and $L_s$ has order $2s$.
The operator inherits the $\mathfrak{se}(d)$ null-space of the elasticity operator. The Green's kernel of $L_s$ is obtained by convolving the Kelvin matrix $\Phi$ with the scalar Sobolev kernel $k_{H^{s-1}}$, i.e.,
\begin{equation*}
    k_s = k_{H^{s-1}} * \Phi
    .
\end{equation*}
Closed form solutions for operators of the form $a(\Delta)L_{\on{sym}}$ for a polynomial $a(\Delta)$ in the Laplacian are identified in \cite{rogulaBasicSolutionsStrain1973}. Their structure is determined by the roots of the polynomial. In the case $s=2$, this polynomial has only the single root $\sigma^{-2}$ and the Green's kernel of the screened elasticity operator is, for $d=2$, 
\begin{align}
&(k_2)_{ij}(x) = 
-\frac{1}{2\pi\mu}\,\delta_{ij} \left(
\log r + K_0(r/\sigma)
\right)
\label{eq:k2_2d}
\\ &\quad
+\left(\frac{1}{2\pi\mu}-\frac{1}{2\pi(\lambda+2\mu)}\right) \partial_i\partial_j \left(
\frac{r^2}{4}\log r - \frac{r^2}{4} + \sigma^2 \left(\log r + K_0(r/\sigma)\right)
\right)
\nonumber
\end{align}
where $K_0$ denotes the modified Hankel function and $r=|x|$.
For $d=3$, the kernel is
\begin{align}
&(k_2)_{ij}(x) =
\frac{1}{4\pi\mu} \delta_{ij} \left( \frac{1}{r} - \frac{e^{-r/\sigma}}{r} \right)
\label{eq:k2_3d}
\\&\quad
+ \left(-\frac{1}{4\pi\mu}+\frac{1}{4\pi(\lambda + 2\mu)} \right)\partial_i \partial_j \left( \frac{r}{2} + \sigma^2 \left( \frac{1}{r} - \frac{e^{-r/\sigma}}{r} \right) \right)
\qquad\qquad\quad
\nonumber
\end{align}
For $s>2$, the polynomial $a(\Delta)$ has multiple roots. As noted in \cite{rogulaBasicSolutionsStrain1973}, the fundamental solutions can then be obtained by perturbing the roots and passing to a limit, or they can be found by an extension of the direct approach pursued in Appendix~\ref{app:A2}.

The smoothing removes the singularity of the Kelvin matrix at $0$ as expected: For $d=3$, $k_s$ is $C^k$ for $k=2(s-1)-2$ and, for $d=2$, $k_s$ is $C^k$ for $k=2(s-1)-1$.
For $r\to\infty$ and $d=3$, $s=2$, the solution differs from the Kelvin matrix in the term
\[
\frac{\sigma^2}{4\pi r^3} \left( 
\frac{1}{\mu}
+ \frac{1}{\lambda + 2\mu} 
\right)
\left( \delta_{ij} - 3\frac{x_i x_j}{r^2} \right)
\] 
plus exponentially decreasing terms, and, for $d=2$, $s=2$, in the term
\[
\frac{\sigma^2}{2\pi r^2} \left( 
\frac{1}{\mu}
+ \frac{1}{\lambda + 2\mu} 
\right)
\left( \delta_{ij} - 2\frac{x_i x_j}{r^2} \right)
\] 
plus exponentially decreasing terms \cite{rogulaBasicSolutionsStrain1973}.
To illustrate vector fields generated by the \(k_2\) kernel, Figure~\ref{fig:k2-kernel-columns} shows the two columns of the kernel matrix evaluated at points in a subset of $\mathbb R^2$.
In particular, $k_s$ does not vanish at infinity for $d=2$, in contrast to the positive definite Sobolev $k_{H^s}$ kernels.

\begin{figure}[t]
    \centering
    \includegraphics[width=.95\textwidth,clip,trim=0 0 0 25]{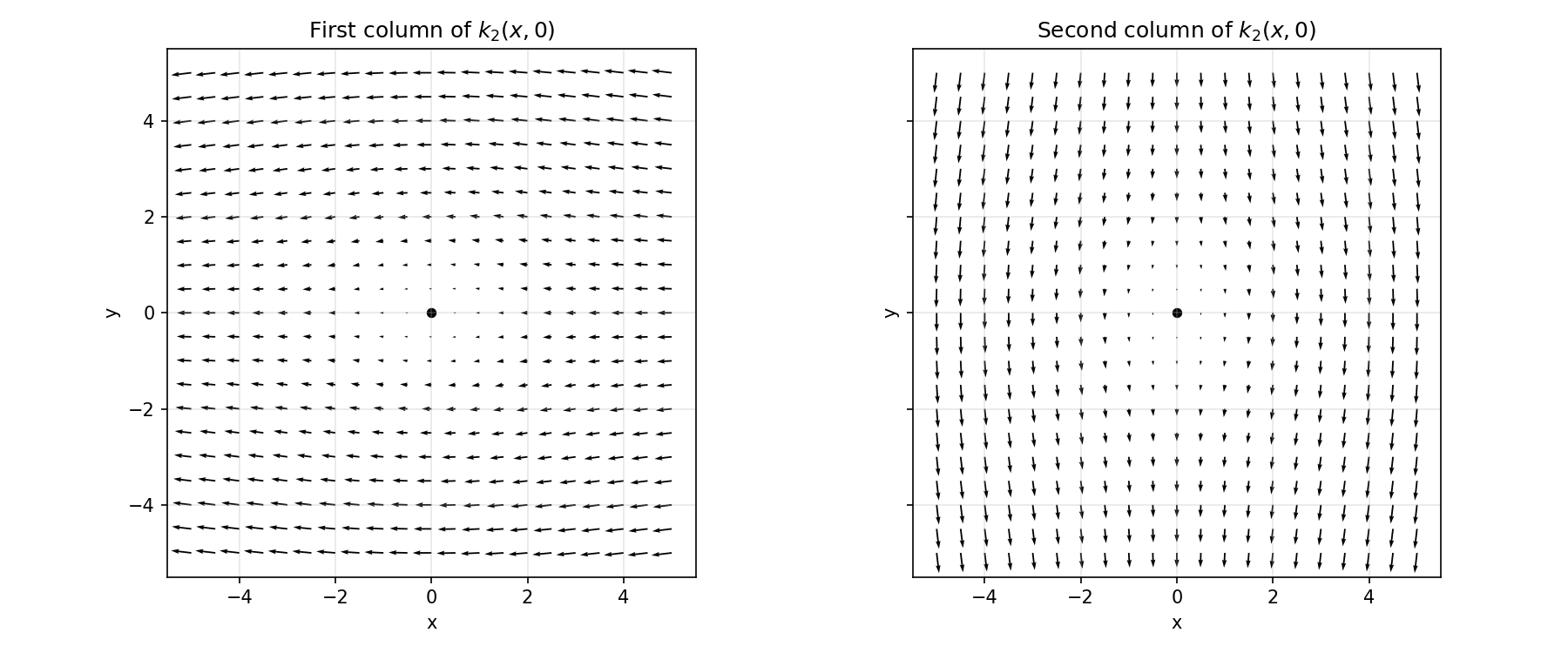}
    \caption{Vector-field visualization of the two columns of the \(k_2\) kernel for $d=2$. Left: the field \(x \mapsto k_2(x,0)e_1\), i.e. the first column of the kernel matrix. Right: the field \(x \mapsto k_2(x,0)e_2\), i.e. the second column. The plots illustrate the directional structure of the \(k_2\) metric and that the fields do not vanish at infinity.}
    \label{fig:k2-kernel-columns}
\end{figure}

\subsection{Interpretation as CPD kernel}
The general CPD construction applies here with the polynomial space replaced by the finite-dimensional space $\mathfrak{se}(d)$ of infinitesimal rigid motions and with the projector $\Pi_{\mathfrak{se}(d)}$ from section~\ref{sec:lagrange_basis}. Thus the screened elasticity kernel $k_s$ is conditionally positive definite with respect to $\mathfrak{se}(d)$: For a landmark configuration $q=(q_1,\ldots,q_n)$, the quadratic form
\begin{equation}
(\al,\be)\mapsto \sum_{i,j=1}^n \langle\al_i,k_s(q_i-q_j)\be_j\rangle =: K_q(\al,\be)
\label{eq:ks_quadratic_form}
\end{equation}
is positive definite on  covectors satisfying
\[
\sum_{i=1}^n \langle\al_i,X(q_i)\rangle=0,
\qquad
\sum_{i=1}^n \langle\be_i,X(q_i)\rangle=0
\qquad
\text{for all }X\in\mathfrak{se}(d).
\]

\subsection{Metrics  on landmark quotients}
\label{ssec:metric_landmark_quotient}
The family of vector subspaces $(\mathfrak{se}(d)(q)^{\mathrm{ann}})_{q\in\on{Land}^n}$ of $T^*\on{Land}^n$ is vector subbundle of different dimensions over each orbit stratum. The kernel quadratic form 
$K_q$ defined in \eqref{eq:ks_quadratic_form} is strictly positive definite on $\mathfrak{se}(d)(q)^{\mathrm{ann}}$ for each $q\in \on{Land}^n$. By the $\on{SE}(d)$-invariance of $k_s$, for each $g\in\on{SE}(d)$ the cotangent map 
$$
T^*_{q}(\on{action}_g): T^*_{g.q}\on{Land}^n \to T^*_{q}\on{Land}^n
$$ 
restricts to an isometry 
$$(\mathfrak{se}(d)(g.q)^{\mathrm{ann}}, K_{g.q}) \to (\mathfrak{se}(d)(q)^{\mathrm{ann}},K_q).$$ 
Recall from section \ref{sec:modSE} that 
$(T_q\pi)^*: T^*_{\pi(q)}\overline{\on{Land}}^n \xrightarrow{\cong} \mathfrak{se}(d)(q)^{\mathrm{ann}}$ is an isomorphism.
Therefore the family of positive definite quadratic forms $(K_q)_{q\in\on{Land}^n}$ induces a stratified cometric $\bar K$ on $T^*\overline{\on{Land}}^n$ whose inverse $\bar K^{-1}$ is a stratified Riemannian metric on 
$\overline{\on{Land}}^n$. We denote this metric $g^{L_s}$.

Let $\pi(q)=\bar q$ and  $\bar v\in T_{\bar q}\overline{\on{Land}}^n$ with $\bar v=\bar{\al}^\flat$ for an $\bar{\al}$ with $(T_q\pi)^*(\bar{\al})=\alpha\in \mathfrak{se}(d)(q)^{\mathrm{ann}}$ and $\,^\flat$ the $g^{L_s}$ Riemannian flat-map. Define the vector field $v^{\on{hor}}(x)=\sum_{i=1}^n k_s(x-q_i)\alpha_i$ on $\R^d$. By construction, $v^{\on{hor}}$ lies in the native space $\mathcal F$ described in \ref{sec:lagrange_basis}, and we have
$$\|\bar v\|_{\bar K^{-1}}^2=K_q(\al,\al)^2=\|v^{\on{hor}}\|_{\Phi}^2=\int_{\R^d}\langle L_s v^{\on{hor}}, v^{\on{hor}}\rangle dx$$ 
so $v^{\on{hor}}$ can be considered a horizontal lift of $\bar v$. The semi-inner product on $\mathcal F$ can be extended to a full RKHS inner product as mentioned in Appendix~\ref{sec:cpd_kernels} so that $v^{\on{hor}}$ is orthogonal to the vertical subspace of infinitesimal rigid motions and thus a horizontal lift. The horizontal lift has the regularity of $k_s$.

\subsection{Geodesics}\label{ssection:geodesics}

The screened elasticity kernel $k_s$ gives a Hamiltonian
\begin{equation}
    H(q,\al)
    =\frac12 K_q(\al,\al)
    =\frac12\sum_{i,j=1}^n \langle \al_i, k_s(q_i-q_j)\al_j\rangle
    \label{eq:ks_Hamiltonian}
\end{equation}
on the cotangent bundle $T^*\on{Land}^n$ similarly to the case for standard positive definite metrics.
If the initial momentum is a pullback of a covector in $T^*_{\pi(q)}\big(\overline{\on{Land}}^n\big)$, then it is in $\mathfrak{se}(d)(q)^{\mathrm{ann}}$, and, as shown in the proposition below, it remains in $\mathfrak{se}(d)(q)^{\mathrm{ann}}$ along the Hamiltonian flow.
\begin{proposition}
    \label{prop:momentum_map_preserved}
For $(\Omega,a)\in \mathfrak{se}(d)$, define
\[
J_{(\Omega,a)}(q,\al):=\sum_{i=1}^n \al_i(\Omega q_i+a)\in \mathbb R.
\]
Then the function $J_{(\Omega,a)}$ is preserved along the Hamiltonian flow of $H$. In particular, if
$\al(0)\in \mathfrak{se}(d)(q(0))^{\mathrm{ann}}$, then $\al(t)\in \mathfrak{se}(d)(q(t))^{\mathrm{ann}}$ for all $t$ for which the flow exists.
\end{proposition}

\begin{proof}
The function $J_{(\Omega,a)}$ is the momentum map for the cotangent-lifted diagonal action of $\SE(d)$ on $\on{Land}^n$. Since $k_s$ is $\SE(d)$-equivariant, the Hamiltonian $H$ is $\SE(d)$-invariant, hence the Poisson bracket
\[
\{J_{(\Omega,a)},H\}=0
\]
for every $(\Omega,a)\in\mathfrak{se}(d)$. Therefore $J_{(\Omega,a)}$ is constant along the Hamiltonian flow by Noether's theorem.
\end{proof}

The Hamiltonian flow lines on $T^*\on{Land}^n$ which start with initial momenta in $\mathfrak{se}(d)(q(0))^{\mathrm{ann}}$ project to horizontal geodesics on $\on{Land}^n$ and hence to geodesics on $\overline{\on{Land}}^n$ as long as they do not cross more singular strata.

In the LDDMM setting with positive definite $C^1$-kernels, the landmarks will not collide along a Hamiltonian  trajectory and the landmark manifold is geodesically complete \cite{trouveMetamorphosesLieGroup2005,bauerOverviewGeometriesShape2014,habermannCharacterizationGeodesicCompleteness2025a}. In the present case, when the kernel is $C^1$, for a Hamiltonian trajectory with $\alpha(0)\in\mathfrak{se}(d)(q(0))^{\mathrm{ann}}$, the horizontal velocity
\begin{equation*}
v_t^{\mathrm{hor}}(x)
=
\sum_{i=1}^n k_s(x-q_i(t))\alpha_i(t)
\end{equation*}
will belong to $C_b^1(\mathbb R^d,\mathbb R^d)$, and it integrates to a flow of $C^1$-diffeomorphisms. Moreover, with $m_t=\sum_i\alpha_i(t)\delta_{q_i(t)}$, the pair $(v_t^{\mathrm{hor}},m_t)$ satisfies the Euler-Poincar\'e equations EPDiff in the distributional sense for the right-invariant quadratic form induced by $L_s$. Since the screening removes the singularity of the Kelvin matrix at the origin, we saw that $k_s$ is $C^1$ with $s\ge\frac52$ for $d=3$ and $s\ge 2$ for $d=2$, i.e. for $s\ge 1+\frac d2$. In these cases, the Hamiltonian trajectory will exist for all time and we have the following result.
\begin{theorem}
\label{thm:non-collision}
Let $d\in\{2,3\}$ and $s\ge1+\frac d2$. If
$q(0)\in\on{Land}^n(\mathbb R^d)$ and
$\alpha(0)\in\mathfrak{se}(d)(q(0))^{\mathrm{ann}}$, then the
Hamiltonian trajectory $(q(t),\alpha(t))$ exists for all time and 
the landmarks will not collide along the trajectory.
\end{theorem}
The theorem is proved in Appendix~\ref{app:non-collision}.

\section{Scale normalization}
\label{sec:scale}
Since the screened elasticity operator is not invariant under rescaling, we here pursue the same approach as in Kendall's shape space to restrict a to fixed-scale sphere. Recall the centroid $C(q)=\sum_i \frac1n q_i$ and the radius function $R(q) = \sum_i |C(q)-q_i|^2$. 
The radius function is invariant for the diagonal action of $\SE(d)$ and thus factors to a function 
$\bar R:\overline{\on{Land}}^n\to \mathbb R$. The space 
\begin{equation*}
\overline{\on{Land}}^n_1(\mathbb R^d) =
\big\{\bar q\in \overline{\on{Land}}^n\mid \bar R(\bar q)=1\big\}\xhookrightarrow{\ }\overline{\on{Land}}^n
\end{equation*}
is a Whitney-stratified subspace with smooth strata which are submanifolds of the strata of $\overline{\on{Land}}^n$. The constraint $\bar R(\bar q)=1$ is a regular equation on each stratum. 

Consider the subspace Riemannian metric on $\overline{\on{Land}}^n_1(\mathbb R^d)$ inherited from $g^{L_s}$. Its inverse defines a Hamiltonian function on the cotangent bundle $T^*\overline{\on{Land}}^n_1(\mathbb R^d)$, and we have the following result.
\begin{proposition}
On the quotient space $\overline{\on{Land}}^n$, restrict the metric $g^{L_s}$ to $ g_1$ on $\overline{ \on{Land}}^n_1(\mathbb R^d)$. The inverse $g_1^{-1}$ then defines a Hamiltonian function on $T^*\overline{\on{Land}}^n_1(\mathbb R^d)$ whose flowlines project to $g_1$-geodesics on $\overline{ \on{Land}}^n_1(\mathbb R^d)$.    
\end{proposition}

Let $\mathfrak n(\bar q)$ be the Riemannian normal to $T_{\bar q}\overline{\on{Land}}^n_1$ in $T_{\bar q}\overline{\on{Land}}^n$.
For $\bar q\in \overline{\on{Land}}^n_1(\mathbb R^d)$, we have 
\begin{align*}
T_{\bar q}\overline{\on{Land}}^n &= T_{\bar q}\overline{\on{Land}}^n_1(\mathbb R^d) \oplus \mathbb R.\mathfrak n(\bar q).
\end{align*}
Let $N(q)$ be a tangent vector satisfying the $\mathfrak{se}(d)$-constraints which projects to $\mathfrak{n}(\bar q)$ up to rescaling in $T_{\bar q}\overline{\on{Land}}^n$. Then $N(q)$ is the constrained $K_q^{-1}$-gradient of the function $R$. An expression for it is derived in Appendix~\ref{app:scale}.

Contrary to the CPD constraints which are preserved by the Hamiltonian flow, the Hamiltonian flow on $T^*_{\bar q}\on{Land}^n$ will not generally preserve tangentiality to $T^*_{\bar q}\on{Land}_1^n(\mathbb R^d)$ because the Poisson bracket between evaluation on $N(q)$ and the Hamiltonian does not vanish, see Appendix~\ref{app:scale}. However, since $R(\bar q)=1$ is a holonomic constraint, we can solve the geodesic equations directly as a constrained Hamiltonian system by introducing a Lagrange multiplier, see \cite[section 8.3]{marsdenIntroductionMechanicsSymmetry1999}.
The resulting constrained Hamiltonian system is given by the equations
\begin{align*}
    \dot q &= \partial_\alpha H(q,\alpha), \\
    \dot\alpha &= -\partial_q H(q,\alpha)-\lambda\,dR(q),
\end{align*}
with $R(q_0)=1$ and $\alpha(N(q_0))=0$.

\section{Numerical implementation}
\label{sec:implementation}
To integrate the geodesic equations numerically, we use automatic differentiation for the kernel and Hamiltonian equations. While the CPD constraints are preserved exactly in the continuous system when satisfied initially, numerical integration will generally violate them. Additionally, in the scale-normalized case, we need $R$ to be preserved when integrating the constrained Hamiltonian system. We therefore use a RATTLE \cite{andersenRattleVelocityVersion1983} type method to preserve the constraints approximately. The scheme is described below.

For matching as in section~\ref{sec:shape_matching}, we use a rigid-motion invariant discrepancy between two configurations $q,q'\in(\mathbb R^d)^n$ based on the pairwise distance matrices with entries
\begin{equation*}
    D(q)_{ij} = \|q_i-q_j\|,
\end{equation*}
and define the loss
\begin{equation}
    \mathcal L(q,q')=\|D(q)-D(q')\|_F^2
    \label{eq:rigid_motion_invariant_loss}
\end{equation}
using the Frobenius norm. The discrepancy being $\SE(d)$ invariant in both $q$ and $q'$ individually implies that we do not need to align the shapes prior or during the matching process. For visual purposes, it can be beneficial to align the shapes after matching to access the results.

We can obtain gradients with respect to $p_0$ using automatic differentiation through the time integration, and the optimization can be gradient-based using e.g. Adam for the optimization together with projection of the iterates to the rigid-motion constraint subspace. 

\subsection{Constraint-preserving time integration}
To integrate the constrained Hamiltonian system numerically, we use a RATTLE-type method adapted to the rigid-motion constraints on the landmark momentum. The aim is to preserve the constraints through the integration.

Given \((q^k,p^k)\), we let each step with step size $\Delta t$ be defined as follows:

\begin{algorithm}
    \caption{Hamiltonian integration step}
    \label{alg:projected-stormer-verlet}
    \begin{algorithmic}[1]
    \Require Current state $(q^k,p^k)$, time step $\Delta t$
    
    \State Compute $p^{k+\frac12}$ implicitly and project the half-step momentum to the constraint subspace at $q^k$:
    \begin{align*}
    &p^{k+\frac12}
    =
    p^k - \frac{\Delta t}{2}\,
    \partial_q H(q^k,p^{k+\frac12})
    \\
    &
    p^{k+\frac12}
    \gets
    \Pi^\perp_{\mathfrak{se}(d)}(q^k)\,p^{k+\frac12}
    \end{align*}
    
    \State Compute $q^{k+1}$ implicitly:
    \[
    q^{k+1}
    =
    q^k + \frac{\Delta t}{2}
    \Bigl(
    \partial_p H(q^k,p^{k+\frac12})
    +
    \partial_p H(q^{k+1},p^{k+\frac12})
    \Bigr)
    \]
    
    \State Compute $p^{k+1}$ by the final half-step momentum update and project the updated momentum to the constraint subspace at $q^{k+1}$:
    \begin{align*}
    &p^{k+1}
    =
    p^{k+\frac12}
    -
    \frac{\Delta t}{2}\,
    \partial_q H(q^{k+1},p^{k+\frac12})
    \\
    &
    p^{k+1}
    \gets
    \Pi^\perp_{\mathfrak{se}(d)}(q^{k+1})\,p^{k+1}
    \end{align*}
    
    \State \Return $(q^{k+1},p^{k+1})$
    \end{algorithmic}
    \end{algorithm}

The orthogonal projection onto the annihilator of the rigid-motion basis $\Pi^\perp_{\mathfrak{se}(d)}(q)$ is concretely
\begin{equation*}
\Pi^\perp_{\mathfrak{se}(d)}(q)
=
I - B(q)\bigl(B(q)^TB(q)\bigr)^{-1}B(q)^T.
\end{equation*}
where \(B(q)\) is the matrix whose columns are the basis vectors
\[
p_i(x)=e_i,
\qquad
p_{ij}(x)=x_j e_i-x_i e_j,
\]
evaluated at the landmarks and flattened into \(\mathbb R^{nd}\).

The two implicit subproblems can be solved in practice by a small fixed number of iterations. The scheme improves both the preservation of the rigid-motion constraints and the Hamiltonian compared to a simple Euler or projected leapfrog method.

\section{Experiments}
\label{sec:experiments}
We now perform numerical experiments to illustrate geodesic evolution as obtained by integrating the Hamiltonian equations, shooting-based matching using the quotient metric, the local rigid motion preservation of the $k_2$ kernel, and the effect of scale normalization.

\subsection{Constrained Hamiltonian evolution}
We first illustrate in Figure~\ref{fig:landmark_hamiltonian_demo} constrained Hamiltonian evolution using random initial momentum in the constraint subspace. As can be seen from the plot, both the Hamiltonian and the constraints are preserved through the evolution with different magnitudes of the initial momentum. The integration is performed on the time interval $[0,2]$ with 500 steps.
\begin{figure}[t]
    \centering
    \includegraphics[width=1.\textwidth]{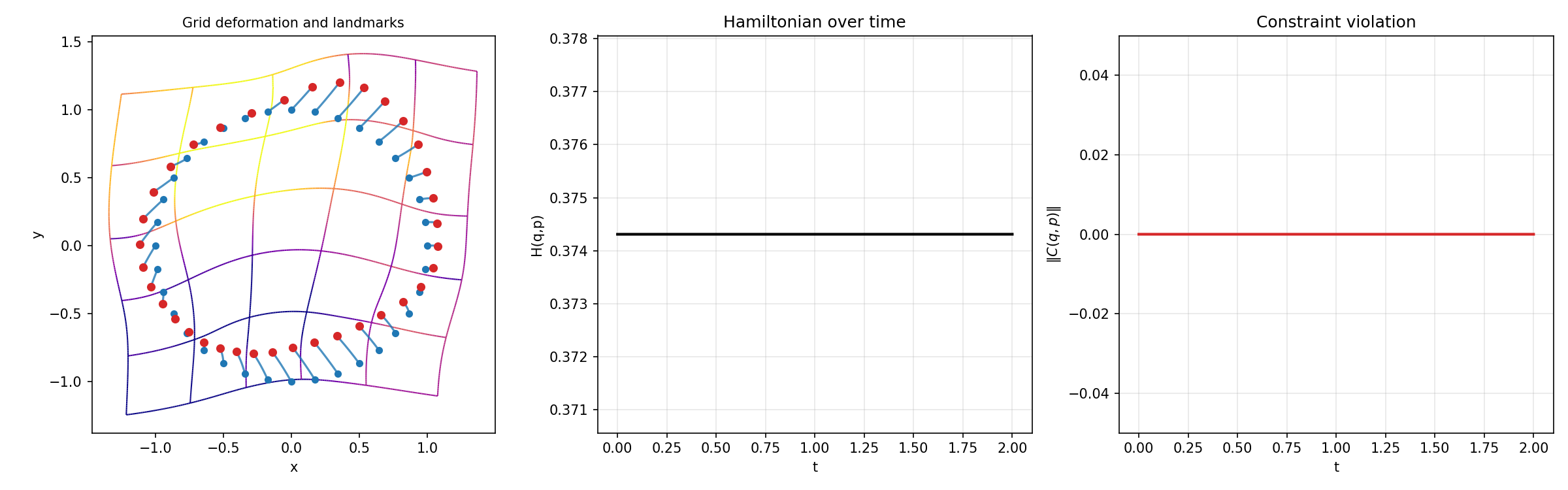}
    \includegraphics[width=1.\textwidth]{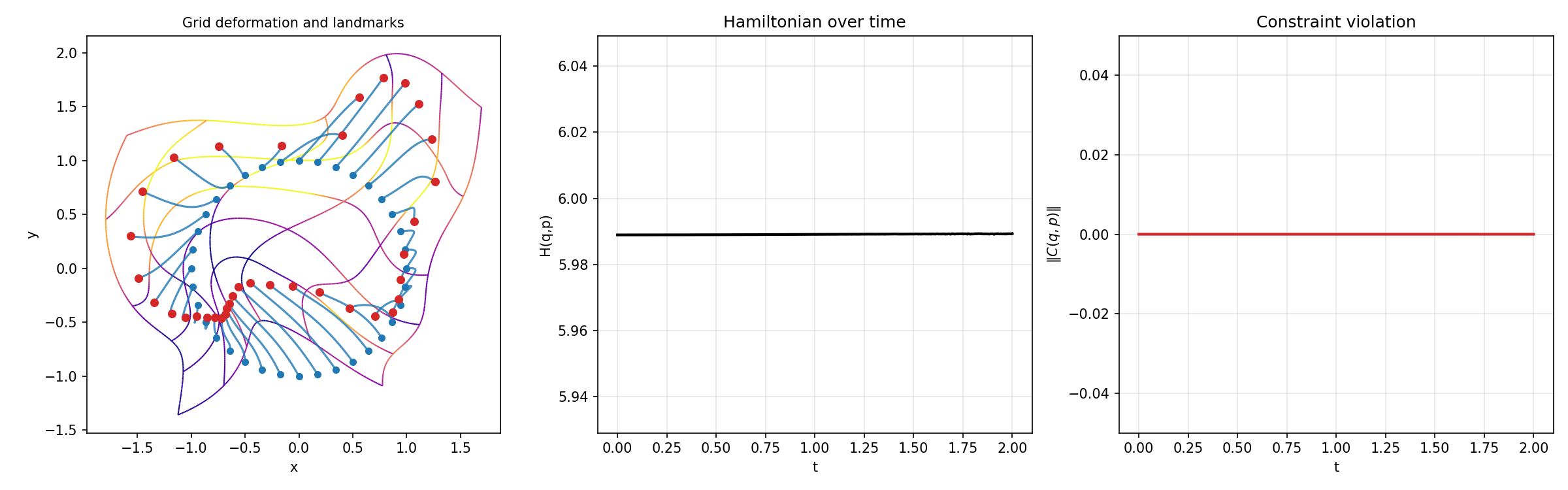}
    \caption{Landmark geodesics with the $k_2$ kernel solved numerically by integrating the Hamiltonian flow. Landmarks initially in a circular configuration (blue points) and random momentum projected to the rigid-motion constraint subspace. Bottom row has $4\times$ larger initial momentum than top row leading to larger deformation.
    Left: Landmark configuration evolution and induced warp of the domain displayed on the underlying grid. Center and right: The Hamiltonian and the constraints are approximately preserved through the integration.}
    \label{fig:landmark_hamiltonian_demo}
\end{figure}

\subsection{Forwards and backwards integration}
To illustrate the stability of the numerical method, in Figure~\ref{fig:landmark_hamiltonian_demo_backwards} we again perform forwards integration with random initial momentum to obtain a final configuration $(q_T,p_T)$. We then start the integration with initial conditions $(q_T,-p_T)$ and simulate the flow again, thus obtaining the backwards flow. The backwards flow can be seen to precisely recover the initial configuration, a consequence of the stability of the numerical scheme. The experiment is performed with a very large initial momentum to demonstrate the stability of the scheme.
\begin{figure}[t]
    \centering
    \includegraphics[width=1.\textwidth]{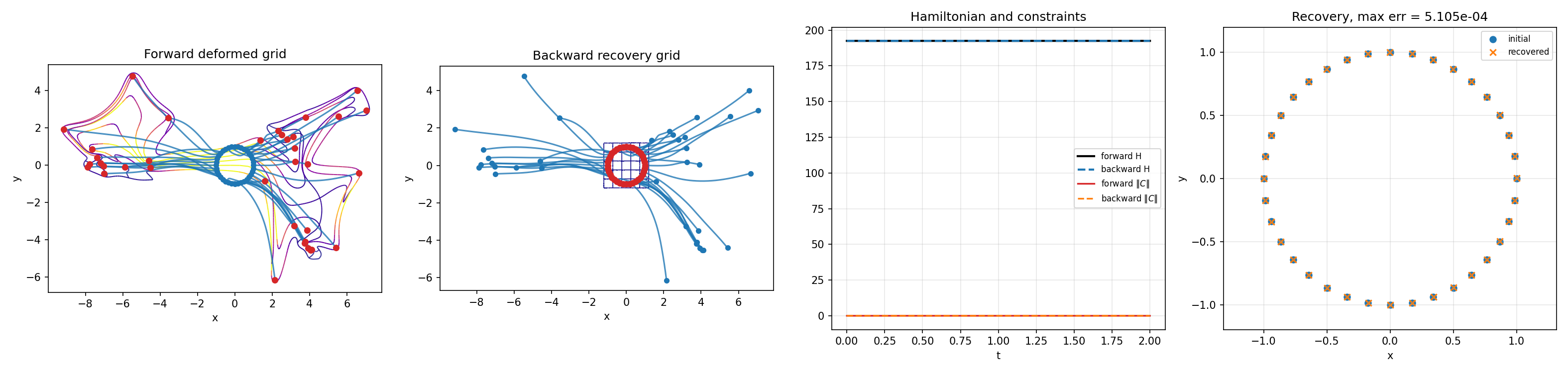}
    \caption{Forwards and backwards integration of the landmark geodesics with the $k_2$ kernel. Left: Forwards integration of the landmark geodesics with the $k_2$ kernel, random initial momentum. The initially rectangular grid is shown deformed by the flow. Center left: Integration with initial conditions are $(q_T,-p_T)$ where $q_T$ is the final configuration and $p_T$ is the momentum from the forwards integration. The backwards flow recovers the circular landmark configuration and the rectangular grid is recovered. Center right: Hamiltonian and constraints are preserved through the integration. Right: Difference between the initial and recovered configurations.}
    \label{fig:landmark_hamiltonian_demo_backwards}
\end{figure}

\subsection{Shooting-based matching}
In the third experiment, we match an initial circular configuration with $n=32$ landmarks to a deformed target configuration. The target is obtained from the circle by anisotropic deformation and translation. We use the rigid-motion invariant loss \eqref{eq:rigid_motion_invariant_loss} to match the initial configuration to the target.

As illustrated in Figure~\ref{fig:landmark_matching_demo}, the shooting procedure is able to match the initial configuration to the target up to rigid motions as intended. Figure~\ref{fig:butterflies} in the introduction illustrates the matching with the $k_2$ kernel of two butterfly shapes with the same setup.

\begin{figure}[t]
    \centering
    \includegraphics[width=\textwidth]{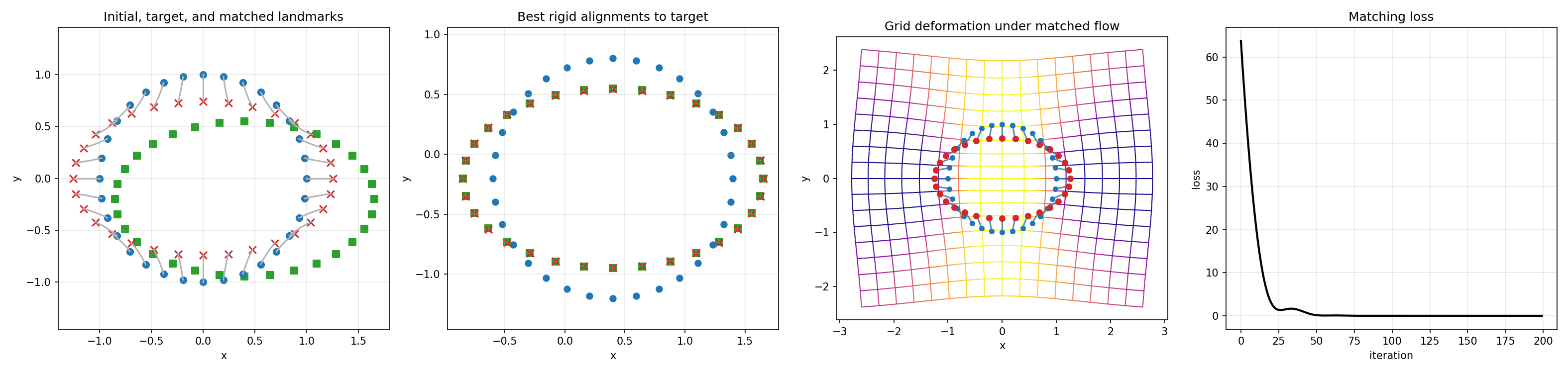}
    \caption{Shooting-based matching produced by optimizing the initial momentum of the constrained Hamiltonian equations. Left: initial landmarks, target landmarks, and the optimized trajectory. Center left: Initial and final configuration, both after best rigid alignment to the target. Center right: Grid deformation induced by the matching. Right: The rigid-motion invariant loss during optimization.}
    \label{fig:landmark_matching_demo}
\end{figure}

\subsection{Local rigid motions}
To illustrate the local rigid motion preservation of the $k_2$ kernel, we perform a similarly matching experiment as above but with two circular configurations matched to translated version in opposite horizontal directions, i.e. the difference between the configurations is a local translation in the horizontal direction. In Figure~\ref{fig:matching-grid-kernel-comparison} in the introduction, we perform matching with both the $k_2$, Gaussian and Mat\'ern-$3/2$ kernels in a case of two circles that are each translated in opposite horizontal directions. The $k_2$ kernel preserves the local rigid motion inside the circles while the other kernels perform interpolation that tend to inhibit local rigid motion because of the identity component in the operator. The movement is thus centered only around the landmarks themselves. This effect is visible in the deformation of the underlying grid.

In Figure~\ref{fig:landmark_matching_local_rigid}, we perform the same matching experiment for two rotated circles with the $k_2$ kernel. Here we again see the local rigid motion inside the circles preserved.

\begin{figure}[t]
    \centering
    \includegraphics[width=\textwidth]{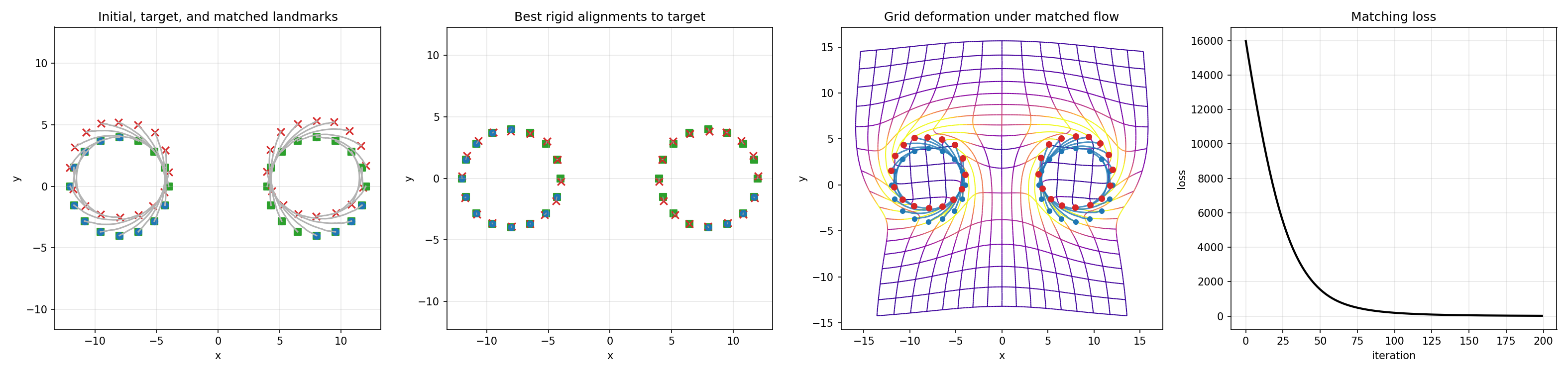}
    \caption{Landmark matching for two circles rotated 90 degrees in opposite directions with the $k_2$ kernel. The local rigid motion inside the circles is preserved.}
    \label{fig:landmark_matching_local_rigid}
\end{figure}

\subsection{Scale normalization}
In Figure~\ref{fig:landmark_matching_scale_normalized}, we perform matching with the $k_2$ kernel with and without scale normalization. Both shapes in the matching satisfy the $R(q)=1$ constraint. As shown in the figure, the scale is preserved during the flow in the scale-normalized case, while not in the non-scale-normalized case.
\begin{figure}[t]
    \centering
    \includegraphics[width=\textwidth]{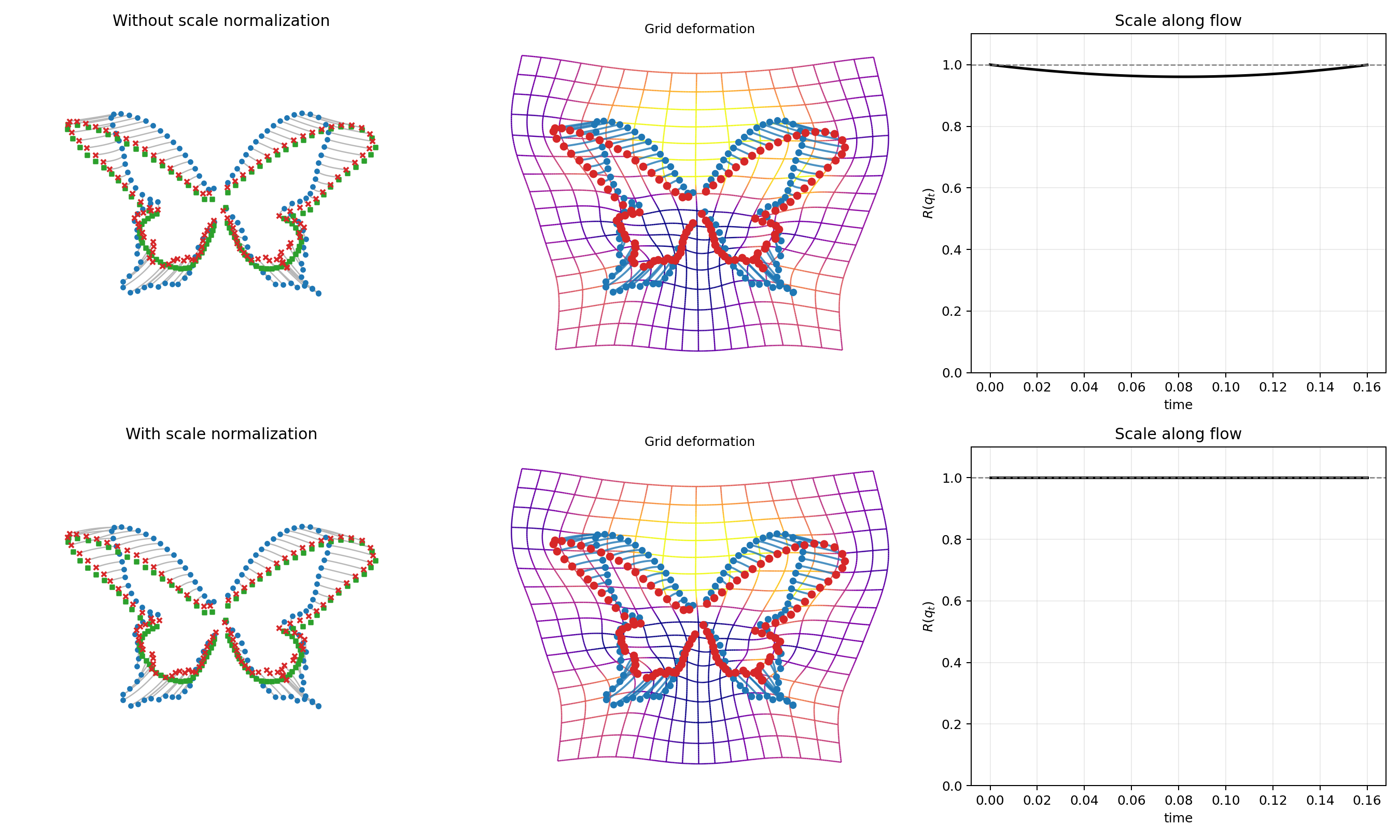}
    \caption{Landmark matching for the butterflies in Figure~\ref{fig:butterflies} with and without scale normalization. The scale is preserved in the scale-normalized case, while not in the non-scale-normalized case.}
    \label{fig:landmark_matching_scale_normalized}
\end{figure}

\section*{Acknowledgements}
We thank Nicolas Charon for suggesting the rigid-motion invariant loss \eqref{eq:rigid_motion_invariant_loss} used in the shape matching experiments. 

The authors used large language models to help deriving and checking some of the mathematical results in the paper, as well as for writing the code for the numerical experiments.

The work presented in this paper was supported by the Villum Foundation Grant 40582, the Novo Nordisk Foundation grants NNF18OC0052000, NNF24OC0093490 and NNF24OC0089608.

\appendix
\section{Appendix}

\subsection{Fourier transform conventions}
\label{sec:fourier_conventions}
We write \(\widehat{\cdot}=\mathcal F[\cdot]\) for the Fourier transform on $\mathbb R^d$ with the convention
\(
\widehat{f}(\xi)=\int_{\mathbb R^d}e^{-ix\cdot\xi}f(x)\,dx\) and \(\mathcal F^{-1}[g](x)=(2\pi)^{-d}\int_{\mathbb R^d}e^{ix\cdot\xi}g(\xi)\,d\xi.
\)
In Fourier variables, we denote the longitudinal/transverse projectors
\[
\Pi^\|(\xi)=\frac{\xi\xi^{\!\top}}{|\xi|^2},\qquad \Pi^\perp(\xi)=\Id-\Pi^\|(\xi).
\]
$\Pi^\|(\xi)$ and $\Pi^\perp(\xi)$ can be seen as the Fourier space equivalents of the Helmholtz projectors $\Prirrot(x)$ and $\Prsol(x)$ in \eqref{eq:k_split}. 

\subsection{Operators and spaces of vector fields}
For $v\in L^2(\R^d,\mathbb R^d)$, let $\hat v$ denote the Fourier transform of $v$. LDDMM most often has its Lie algebra being a Sobolev space
\begin{equation}
    H^s(\R^d,\R^d) 
    = 
    \left\{V\in L^2(\R^d,\R^d)\,|\,\|v\|_s^2:=\int(1+\sigma^2|\xi|^2)^s|\hat{v}(\xi)|^2d\xi<\infty\right\}
    . 
\end{equation}
Equivalently, for integer $s$, we can start with the differential operator $L_{H^s}=(\Id-\sigma^2\Delta)^s$, define the norm $\|v\|_V^2=\left<Lv,v\right>_{L^2}$ and retrieve $H^s$ as the  subspace of $L^2(\R^d,\R^d)$ where this norm is finite. Here the Laplacian is $\De=\sum_{i=1}^d\partial_{x_i}^2$, and the exponent $s$ in $L_{H^s}$ means $s$ times application of the operator $\Id-\sigma^2\Delta$. For vector fields, the application is componentwise. 
In geometric mechanics and LDDMM, the Sobolev operators are often called inertia operators or momentum operators. The inverse of the inertia operator is denoted the kernel and is denoted by $K$.

The Laplacian operator $L_{\Delta^s}=(-1)^s\De^s$ uses $s$ iterated applications of the Laplacian $\De^s=\De\De^{s-1}$. It is another example of a differential operator, lacking the identity component of $L_{H^s}$; in this sense, $L_{H^s}$ may be regarded as a screened version of $L_{\Delta^s}$, analogously to the screened Poisson equation. The Banach space associated to $L_{\Delta^s}$ is the Beppo-Levi space  $\dot{W}^{s,p}$ of functions with all derivatives of order $s$ in $L^p(\R^d,\R^d)$, i.e., 
\begin{equation*}
\dot{W}^{s,p} = \{v\in L_{loc}^1(\R^d)\,|\,\|\partial^\al v\|_{L^p(\R^d)}<\infty \text{ for all }\al \text{ with }|\al|= s\}.
\end{equation*}
The norm $\|v\|_{\dot{W}^{s,p}}=\sum_{|\al|=k}\frac{s!}{\al!}\|\partial^\al v\|_{L^p(\R^d)}$ vanishes on the nontrivial set of all $v$ with $\|v\|_{\dot{W}^{s,p}}=0$. 
For $p=2$, we also write $\dot W^{s,2} = \dot H^s$ and we have \begin{equation*}
\|v\|^2_{\dot{W}^{s,2}}=\left<(-1)^s\De^s v,v\right>_{L^2(\R^d)}=\sum_{|\al|=k}\frac{s!}{\al!}\|\partial^\al v\|^2_{L^2(\R^d)}
\end{equation*}
Because of the lack of the identity term, $L_{\Delta^s}$ has a non-trivial null-space, including at least $0$- and $1$st order polynomials, i.e affine transformations of $\R^d$. The operator therefore cannot be inverted directly leading to the notion of conditionally positive kernels described below. The use of Beppo-Levi spaces for shape analysis goes back to Bookstein \cite{booksteinMorphometricToolsLandmark1997}. \cite{boufadeneFastLargeDeformation2025} uses particular Beppo-Levi spaces for landmark matching with the energy distance kernel.

\subsection{Translation- and rotation-invariant operators and kernels}
\label{sec:kernels}
All operators described in the paper are translation- and rotation-invariant, and we therefore describe the general form of the kernels and Fourier transforms of such operators. Note that invariance here concerns rigid motions acting on the domain of the operator, and it is a different concept compared to infinitesimal rigid motions being in the null-space of the operators.

Let $L$ be an operator on vector fields $v\in\mathfrak X(\R^d)$, $\|v\|_L^2=\left<Lv,v\right>_{L^2(\R^d)}$ its corresponding Hilbert space norm, and let $K$ be its kernel on integral form $K:\R^d\to\R^{d\times d}$ so that $LK(\cdot,x)=\delta_x$. Then, if translations $x\mapsto x+b$ act as isometries for the norm, the kernel can be shown to be on the form $K(x,y)=k(x-y)$ for a function $k:\R^d\to\R^{d\times d}$, see e.g. \cite{micheliMatrixvaluedKernelsShape2014}. If furthermore rotations $x\mapsto Rx$ act as $k(Rx)=Rk(x)R^{-1}$, then $k$ can be written as
\begin{equation}
    k(x) = k^\parallel(r)\Prirrot_x + k^\perp(r)\Prsol_x
    \label{eq:k_split}
\end{equation}
for $x\not=0$ where $\Prirrot_x=\frac{xx^T}{|x|^2}$, $\Prsol_x=\Id_d - \frac{xx^T}{|x|^2}$ where $r=|x|$ is the $\R^d$-norm for $x$. $\Prirrot_x$ and $\Prsol_x$ are called the Helmholtz projectors.

When furthermore $k(x)=\tilde{k}(r)\Id_d$, for a scalar function $\tilde{k}$, $k$ is called a scalar kernel and $\tilde{k}$ are often identified with $k$. For the Sobolev operators, the kernels are scalar and given by
\begin{equation}
    k_{H^s}(r)=\frac{\alpha}{2^{s-1}(2\pi)^{\frac{d}{2}}\Gamma(s)\sigma^d}\left(\frac{r}{\sigma}\right)^\nu K_\nu\left(\frac{r}{\sigma}\right)
    \label{eq:k_H^s}
\end{equation}
where $\nu=s-d/2$ and $K_\nu$ denotes the modified Bessel function of order $\nu$; see, e.g., \cite{adamsFunctionSpacesPotential1999,rasmussenGaussianProcessesMachine2008,micheliMatrixvaluedKernelsShape2014}. These kernels are denoted Mat\'ern or Bessel kernels. For half-integer $\nu$, they have explicit expressions that are often used in computations.

The kernel for the Laplacian $L_{\Delta^s}$ of any order also has explicit expressions, see e.g. \cite{abatangeloHigherorderFractionalLaplacians2021} or \cite[(10.11)]{wendlandScatteredDataApproximation2004}. For $s>d/2$ and $x\in \mathbb R^d$ with $r=|x|$, the fundamental solutions are
$$
k_{\Delta^s}(r) = 
\begin{cases} 
\qquad\frac{\Ga(d/2-s)}{2^{2s}\pi^{d/2}\Ga(s)}r^{2s-d} & \quad \text{ for } s-\frac{d}{2} \notin \mathbb N  , \\
\frac{(-1)^{s+(d-2)/2}}{2^{2s-1}\pi^{d/2}\Ga(s)(s-d/2)!}r^{2s-d}\log(r) &
\quad\text{ for } s-\frac{d}{2}  \in \mathbb N .
\end{cases}
$$
Notice that $k_{\Delta^s}$ are not in $L_2(\R^d,\R^d)$.

\subsection{Conditionally positive definite reproducing kernels}
\label{sec:cpd_kernels}
For operators such as the Sobolev operators $L_{H^s}$, the null-space is trivial, the corresponding kernel is positive definite, and the Hilbert space has the structure of a reproducing kernel Hilbert space (RKHS). Since the landmark Riemannian metric will depend on the kernel, to treat operators with non-trivial null-space such as $L_{\Delta^s}$, we need to consider conditionally positive definite kernels.

While the quadratic form $\left<v,w\right>=\left<L_{\Delta^s} v,w\right>_{L^2(\R^d)}$ defines a non-degenerate quadratic form on $\X_\mathcal A(\R^d)$, the completion under the norm will contain elements that cannot be interpreted as continuous functions as is usually the case in an RKHS. However, the notion of conditionally positive definite kernels allows to give a RKHS-like structure in the so-called native space of the kernel, thus giving a way to recover the Beppo-Levi spaces $\dot H^s$ as native spaces in a similar way as can be done for Sobolev spaces. The description here follows \cite{wendlandScatteredDataApproximation2004} in the generality needed for the paper.
\begin{definition}
Let $\mathcal P$ be a subspace of polynomials on $\R^d$, and $k:\R^d\to \R$ a symmetric function such that for any configuration of distinct points $x_1,\ldots,x_N\in \R^d$, the quadratic form
\begin{equation*}
    a,b\mapsto
    \sum_{i,j=1}^N
    k(x_i-x_j)a_i^Tb_j
\end{equation*}
is positive definite when $a,b$ satisfy
\begin{equation}
    \sum_{i=1}^Na_ip(x_i)=0,\ 
    \sum_{i=1}^Nb_ip(x_i)=0
    \label{eq:cpd_constraints}
\end{equation}
for all $p\in \mathcal P$. Then $k$ is said to be conditionally positive definite with respect to $\mathcal P$. If $\mathcal P$ is the subspace of polynomials of order $\le s-1$, then $k$ is said to be conditionally positive definite (CPD) of order $s$.
\end{definition}
The kernels $k_{\Delta^s}$ described above are conditionally positive definite of order $m=s-\lceil d/2\rceil+1$.

The difference from regular positive definite kernels lies in the restriction \eqref{eq:cpd_constraints} to the constraint space. Let $\tilde{\mathcal F}$ denote the space of finite constraint-satisfying linear combinations of the kernel 
\begin{equation*}
\tilde{\mathcal F}=\left\{\sum_{i=1}^R a_i k(\cdot-x_i)
\,\Big|\,
x_i\in\R^d
\,,\,
\sum_{i=1}^R a_i p(x_i)=0\,\forall p\in \mathcal P
\right\}\ .
\end{equation*}
$\tilde{\mathcal F}$ is a pre-Hilbert space with the kernel norm $\|f\|^2=\sum_{i,j}\langle a_i, k(x_i-x_j)a_j \rangle$, and it can be completed to a Hilbert space $\mathcal F$. However, contrary to the positive definite case, generally $k(\cdot-x)\not\in \mathcal F$, i.e. the kernel is not in the RKHS completion. This prevents the definition of evaluation functionals in the completion using the usual RKHS construction, which would look like 
\begin{equation*}
\delta_x f
=
\left<f,k(\cdot-x)\right>_{\mathcal{F}}
\ .
\end{equation*}
Instead,  a projection $\Pi:C(\R^d)\to\mathcal P$ from continuous functions to polynomials can be defined through the choice of a fixed set $x_1,\ldots,x_r$ of reference points that allow to distinguish polynomials. The projection then arises from a Lagrange basis of polynomials at these points. This allows to define functions
\begin{equation*}
    G(\cdot-x)=k(\cdot-x)-\Pi(k(\cdot-x))
\end{equation*}
that are automatically in $\mathcal F$ by construction. This gives a continuous evaluation-like functional
\begin{equation*}
    \delta_{(x)}f=\left<f,G(\cdot-x)\right>_{\mathcal{F}}
\end{equation*}
that can be used to map $\mathcal F$ to $C(\R^d)$ and interpret the Hilbert completion of the kernels as a space of functions. When defining $G$, the projection to polynomials was subtracted. Adding polynomials back as a direct product now gives the native space
\begin{equation*}
    \mathcal N(\R^d)
    =\mathcal F+\mathcal P
\end{equation*}
where elements of $\mathcal F$ are interpreted as functions.  The Beppo-Levi spaces are exactly native spaces $\mathcal N(\R^d)$ for the conditionally positive kernels $k_{\Delta^s}$. 

There is a semi-inner product on $\mathcal N(\R^d)$
\begin{equation*}
    \left<f,h\right>_{\mathcal N(\R^d)}
    =
    \left<f-\Pi(f),h-\Pi(h)\right>
    \ .
\end{equation*}
The splitting between $\mathcal F$ and $\mathcal P$ in the definition of the native space is unique, and the construction including the inner product is independent of the choice of evaluation points in the construction of the projection $\Pi$. The semi-inner product can be extended to an inner product making $\mathcal N(\R^d)$ an RKHS by adding a metric on $\mathcal P$.

\subsection{The screened elasticity kernels}\label{app:A2}
Following \cite{rogulaBasicSolutionsStrain1973}, for $d\in\{2,3\}$, we consider an operator of the form
\begin{equation*}
L=-a(\Delta)\,\Delta-(b(\Delta)-a(\Delta))\,\nabla\operatorname{div},
\label{eq:pole_operator_general}
\end{equation*}
where $a$ and $b$ are scalar polynomials satisfying $a(0)=\mu$ and $b(0)=\lambda+2\mu$. Let $k$ denote the Green's kernel of the operator. In Fourier variables,
\begin{equation*}
\widehat{k}(\xi)
=\frac{\Pi^\perp(\xi)}{|\xi|^2 a(-|\xi|^2)}+\frac{\Pi^\|(\xi)}{|\xi|^2 b(-|\xi|^2)}
=
\frac{\delta_{ij}}{|\xi|^2 a(-|\xi|^2)}
-\frac{\xi_i\xi_j}{|\xi|^4 a(-|\xi|^2)}
+\frac{\xi_i\xi_j}{|\xi|^4 b(-|\xi|^2)}.
\label{eq:pole_symbol_general}
\end{equation*}
We assume first that the roots of $a$ and $b$ are simple and real. Then
\begin{equation*}
a(-\kappa^2)=\mu\prod_{m=1}^p \Bigl(1+\frac{\kappa^2}{\beta_m^2}\Bigr),
\qquad
b(-\kappa^2)=(\lambda+2\mu)\prod_{n=1}^q \Bigl(1+\frac{\kappa^2}{{\beta'_n}^2}\Bigr),
\label{eq:ab_factorization}
\end{equation*}
with $\beta_m,\beta'_n>0$, and we have the partial fraction expansions
\begin{equation*}
\frac{1}{a(-\kappa^2)}=\frac1\mu\sum_{m=1}^p \frac{\alpha_m\beta_m^2}{\kappa^2+\beta_m^2},
\qquad
\frac{1}{b(-\kappa^2)}=\frac1{\lambda+2\mu}\sum_{n=1}^q \frac{\alpha'_n{\beta'_n}^2}{\kappa^2+{\beta'_n}^2}
\end{equation*}
with
$\alpha_m=\prod_{t\neq m}\frac{\beta_t^2}{\beta_t^2-\beta_m^2}$,
$\alpha'_n=\prod_{t\neq n}\frac{{\beta'_t}^2}{{\beta'_t}^2-{\beta'_n}^2}$.
Here $\sum_{m=1}^p \alpha_m=\sum_{n=1}^q \alpha'_n=1$, and 
\begin{align*}
&\frac1{\kappa^2 a(-\kappa^2)}
=
\frac1\mu\sum_{m=1}^p \alpha_m\left(\frac1{\kappa^2}-\frac1{\kappa^2+\beta_m^2}\right),
\\
&\frac1{\kappa^2 b(-\kappa^2)}
=
\frac1{\lambda+2\mu}\sum_{n=1}^q \alpha'_n\left(\frac1{\kappa^2}-\frac1{\kappa^2+{\beta'_n}^2}\right).
\end{align*}
Define
\[
A=\mathcal F^{-1}\left(\frac{1}{|\xi|^2 a(-|\xi|^2)}\right),
\quad
B=\mathcal F^{-1}\left(\frac{1}{|\xi|^2 b(-|\xi|^2)}\right),
\]
and $-\Delta U_a=A$, $-\Delta U_b=B$, respectively. Then
\[
k_{ij}=\delta_{ij}A+\partial_i\partial_j U_a-\partial_i\partial_j U_b.
\]
Finding $k$ reduces to finding the four scalar functions
$k^1=\mathcal F^{-1}\left(\frac1{|\xi|^2}\right)$,
$k^2_\beta=\mathcal F^{-1}\left(\frac1{|\xi|^2+\beta^2}\right)$,
$k^3=\mathcal F^{-1}\left(\frac1{|\xi|^4}\right)$, and 
$k^4_\beta=\mathcal F^{-1}\left(\frac1{|\xi|^2(|\xi|^2+\beta^2)}\right)$,/a
because
\[
A(x)=\frac1\mu\sum_{m=1}^p \alpha_m\left(k^1(x)-k^2_{\beta_m}(x)\right),
\qquad
B(x)=\frac1{\lambda+2\mu}\sum_{n=1}^q \alpha'_n\left(k^1(x)-k^2_{\beta'_n}(x)\right)
\]
and $\mathcal F(U)(\xi)=|\xi|^{-2} \mathcal F(-\Delta U)(\xi)$.
These functions are the Sobolev and Beppo-Levi kernels as described in section~\ref{sec:background}. In $2d$, we have
$k^1(r)=-\frac{1}{2\pi}\log r$,
$k^2_\beta(r)=\frac{1}{2\pi}K_0(\beta r)$,
$k^3(r)=-\frac{1}{8\pi}\left(r^2\log r-r^2\right)$,
$k^4_\beta(r)=-\frac{1}{2\pi\beta^2}\left(\log{r}-K_0(\beta r)\right)$,
and, in $3d$, 
$k^1(r)=\frac{1}{4\pi r}$, 
$k^2_\beta(r)=\frac{e^{-\beta r}}{4\pi r}$,
$k^3(r)=-\frac{r}{8\pi}$,
$k^4_\beta(r)=\frac{1}{4\pi\beta^2}\left(\frac1r-\frac{e^{-\beta r}}{r}\right)$.
This gives, in $2d$,
\begin{align*}
k_{ij}(x)
&=-\frac{1}{2\pi\mu}\,\delta_{ij}
\sum_{m=1}^p \alpha_m\bigl(\log r+K_0(\beta_m r)\bigr)
\nonumber\\
&\quad+\frac{1}{2\pi\mu}\,\partial_i\partial_j
\left(
\frac{r^2}{4}\log r-\frac{r^2}{4}
+
\sum_{m=1}^p\frac{\alpha_m}{\beta_m^2}\bigl(\log r+K_0(\beta_m r)\bigr)
\right)
\nonumber\\
&\quad-\frac{1}{2\pi(\lambda+2\mu)}\,\partial_i\partial_j
\left(
\frac{r^2}{4}\log r-\frac{r^2}{4}
+
\sum_{n=1}^q\frac{\alpha'_n}{{\beta'_n}^2}\bigl(\log r+K_0(\beta'_n r)\bigr)
\right).
\end{align*}
and, in $3d$, 
\begin{align*}
k_{ij}(x)
&=\frac{1}{4\pi\mu}\,\delta_{ij}
\left(
\frac1r-\sum_{m=1}^p \alpha_m\frac{e^{-\beta_m r}}{r}
\right)
\nonumber\\
&\quad-\frac{1}{4\pi\mu}\,\partial_i\partial_j
\left(
\frac r2+
\sum_{m=1}^p\frac{1}{\beta_m^2}\left(\frac1r-\alpha_m\frac{e^{-\beta_m r}}{r}\right)
\right)
\nonumber\\
&\quad+\frac{1}{4\pi(\lambda+2\mu)}\,\partial_i\partial_j
\left(
\frac r2+
\sum_{n=1}^q\frac{1}{{\beta'_n}^2}\left(\frac1r-\alpha'_n\frac{e^{-\beta'_n r}}{r}\right)
\right)
\end{align*}
and hence the result of \cite{rogulaBasicSolutionsStrain1973}. For the screened elasticity family,
\[
a(-\kappa^2)=\mu(1+\sigma^2\kappa^2)^{s-1},
\qquad
b(-\kappa^2)=(\lambda+2\mu)(1+\sigma^2\kappa^2)^{s-1}.
\]
Thus $a$ and $b$ have degree $s-1$, and the corresponding operator has order $2+2(s-1)=2s$. For $s=1$ there is no screening root and $k_1$ is the Kelvin matrix. For $s=2$ there is one simple root $\beta_1=\beta'_1=\sigma^{-1}$, which gives \eqref{eq:k2_2d} and \eqref{eq:k2_3d}. For $s>2$, this root has multiplicity $s-1$, and the same argument applies using the partial fraction expansion for a repeated root

\subsection{Groups of diffeomorphisms}
\label{sec:groups}
Some regular Lie groups of diffeomorphisms on $\mathbb R^d$ are
\begin{itemize}
\item $\Diff_{\mathcal B}(\mathbb R^d)$, the group of all diffeomorphisms which differ from the identity by a function which is bounded together with all derivatives separately; i.e.,
$$\Diff_{\mathcal B}(\mathbb R^d)=\{ \on{Id}_{\mathbb R^d} + f: f\in\mathcal{B}(\mathbb R^d,\mathbb R^d), \det(df) > -1\}\,.$$
Each $\ph\in\Diff_{\mathcal B}(\mathbb R^d)$ is a Lipschitz-map with Lipschitz inverse. 
$\Diff_{\mathcal B}(\mathbb R^d)$ contains $\SE(d)$ and the scaling group $\mathbb R_{>0}.\on{Id}_{\mathbb R^d}$. 
\item $\Diff_{H^\infty}(\mathbb R^d)$, the group of all diffeomorphisms which differ from the identity by a function in the intersection $H^\infty$ of all Sobolev spaces $H^s$ for $k\in \mathbb N_{\ge 0}$.
\item $\Diff_{\mathscr S}(\mathbb R^d)$, the group of all diffeomorphisms which fall rapidly to the 
       identity.
\item $\Diff_c(\mathbb R^d)$, the group of all diffeomorphisms which differ from the identity only on a compact set. 
\end{itemize} 

The four sets of diffeomorphisms $\Diff_c(\mathbb R^d)$, 
$\Diff_{\mathscr S}(\mathbb R^d)$,
$\Diff_{H^\infty}(\mathbb R^d)$, and 
$\Diff_{\mathcal B}(\mathbb R^d)$
are all smooth regular Lie groups with smooth injective group homomorphisms
$$
\Diff_c(\mathbb R^d) \to  \Diff_{\mathscr S}(\mathbb R^d) \to \Diff_{H^\infty}(\mathbb R^d) \to  \Diff_{\mathcal B}(\mathbb R^d) ,
$$
see \cite[3.3]{michorZooDiffeomorphismGroups2013}. Each group is a normal subgroup in any other in which it is contained, in particular in $\on{Diff}_{\mathcal B}(\mathbb R^d)$.  Moreover, $\Diff_c(\mathbb R^d)$ is dense in $\Diff_{\mathscr S}(\mathbb R^d)$ and in $\Diff_{H^\infty}(\mathbb R^d)$, but neither is dense in the rightmost group $\Diff_{\mathcal B}(\mathbb R^d)$.
In \cite{michorZooDiffeomorphismGroups2013} one can find many more regular Lie subgroups of $\Diff_{\mathcal B}(\mathbb R^d)$. 

The constructions in this paper work for any of the subgroups since they all consist of vector fields decaying fast enough near infinity, so that integrals like $\int_{\mathbb R^n}\langle LX,Y\rangle dx$ are finite for the operators $L$ considered here. But they do not work for $\Diff_{\mathcal{B}}(\mathbb R^d)$.  We shall write $\Diff_{\mathcal{A}}(\mathbb R^d)$ where $\mathcal{A}\in \{ H^\infty,\mathcal S, c\}$.

\subsection{Induced metrics on shape spaces}
\label{sec:induced}
The diffeomorphism groups $\Diff_{\mathcal{A}}(\R^d)$ act on a range of shape spaces $\mathcal S$. For landmark configurations $q=(q_1,\dots,q_n)$, curves $q:\SS^1\to \R^d$, or surfaces $q:\SS^2\to \R^d$, the action is by composition $\ph.q=\ph\circ q$. When the group acts transitively, we can regard $\mathcal S$ as an orbit space. If $\Diff_{\mathcal{A}}(\R^d)$ is furthermore equipped with a right-invariant metric, we can fix a shape $s_0\in\mathcal S$ and equip $\mathcal S$ with the induced Riemannian metric that makes the action on $s_0$ a Riemannian submersion. We here describe this construction and the resulting metric specifically for the case of landmark configurations following \cite[Section 9.6]{michorManifoldsMappingsContinuum2020} and \cite{micheliSectionalCurvatureTerms2012}.

A landmark configuration $q\in\on{Land}^n$ can be considered a map $\{1,\dots,n\}\to \mathbb R^d$ and the composition action of diffeomorphisms is then $\ph.q=\ph\circ q=(\ph(q_1),\ldots,\ph(q_n))$. We also write this action as the map $\on{ev}:\Diff_{\mathcal A}(\mathbb R^d)\times \on{Land}^n\to \on{Land}^n$, $\on{ev}(\ph,q)=\ph.q$.
Now let $q_0=(q_{0,1},\dots,q_{0,n})$ be a fixed landmark configuration. Fixing $q_0$ in the evaluation map, we have the surjective mapping $\on{ev}_{q_0}(\ph)=\on{ev}(\ph,q_0)$.

We then consider a right invariant weak Riemannian metric on $\Diff_{\mathcal A}(\mathbb R^d)$ which corresponds to an \emph{injective} linear inertia operator $L$  from the Lie algebra of vector fields $\X_{\mathcal{A}}(\mathbb R^d)$ to its locally convex dual space $\X_{\mathcal{A}}(\mathbb R^d)'$, concrete examples of such operators being the Sobolev operators $L_{H^s}$. The metric $g^L$ is then given by
\begin{equation}
    g^L_\ph(v\o\ph,w\o\ph) = \int_{\mathbb R^d} \langle Lv,w\rangle dx\,,
    \label{eq:gL}
\end{equation}
where $\langle\;,\;\rangle$ denotes the $\mathbb R^d$ inner product. The fiber of $\on{ev}_{q_0}$ over a landmark $q=\ph_0(q_0)$ is  
\begin{align*}
\{\ph\in\on{Diff}_{\mathcal A}(\mathbb R^d): \ph(q_0)=q\}
&= \ph_0\o\{\ph\in\on{Diff}_{\mathcal A}(\mathbb R^d): \ph(q_0)=q_0\}
\\&
= \{\ph\in\on{Diff}_{\mathcal A}(\mathbb R^d): \ph(q)=q\}\o\ph_0\,.
\end{align*}
The tangent space with foot point $\ph_0$ to the fiber over $q$ is 
$$
\{v\o \ph_0: v\in \X_{\mathcal A}(\mathbb R^d), v(q_{0,i}) = 0 \text{  for all } i \}.
$$
A tangent vector $w\o\ph_0 \in T_{\ph_0}\on{Diff}_{\mathcal A}(\mathbb R^d)$ is 
$g^L_{\ph_0}$-perpendicular to the fiber over $q$ if and only if
$$
\int_{\mathbb R^d} \langle Lw, v \rangle\,dx =0\quad 
\forall v\in \X_{\mathcal{A}}(\mathbb R^d)\text{  with }v(q)=0.
$$
Set $L w=\sum_{i=1}^n p_i\otimes \de_{q_i}$, a distributional 1-form with support at the landmark $q$, which we may identify as the pullback via $\on{ev}_{q_0}^*$ of a cotangent vector in $T_q^*\on{Land}^n$. View $p_i\otimes \de_{q_i}\in T_{q_i}^*\mathbb R^d$. We can identify $T_q^*\on{Land}^n_0$ with $(\R^d)^n$ using the Euclidean inner product $\langle\;,\;\rangle$. 
Then 
\begin{equation*}
w(x) = L^{-1}\Big(\sum_{i=1}^n p_i\otimes\de_{q_i} \Big) 
= \int_{\mathbb R^d} k(x-y)\sum_{i=1}^n p_i\otimes\de_{q_i}(y)\,dy
= \sum_{i=1}^n k(x-q_i)p_i,
\end{equation*}
where $k$ is the kernel of $L$, see section \ref{sec:kernels}. This works if the kernel $k$ can be continuously extended to $\sum_{i=1}^n p_i\otimes\de_{q_i}$ as is possible in the RKHS setting, i.e., for a positive definite kernel $k$.  

Now let us consider a tangent vector $v=(v_1,\dots,v_n)\in T_q\on{Land}^n$. Its
horizontal lift with foot point $\ph_0$ is $v^{\text{hor}}\o\ph_0$
where the vector field $v^{\text{hor}}$ on $\mathbb R^d$ is given as follows: Let $K(q)_{ij}$ be the kernel block matrix with $d\times d$ block entries $K(q)_{ij}=k(q_i-q_j)$ and let $v_i$ denote the $d$-subvector of $v$ corresponding to the landmark $q_i$. Then 
\begin{align*}
v^{\text{hor}}(x) &= \sum_{i=1}^n k(x-q_i)(K\i(q)v)_i\,,
\\
L(v^{\text{hor}}(x))&= \sum_{i=1}^n \de(x-q_i)(K\i(q)v)_i\,.
\end{align*}
The Riemannian metric on the finite dimensional manifold $\on{Land}^n$ induced by the $g^L$-metric on
$\on{Diff}_{\mathcal S}(\mathbb R^d)$ such that $\on{ev}_{q_0}$ becomes a Riemannian submersion is then given by
\begin{align*}
g^L_q(v,w) &= g^L_{\ph_0}(v^{\text{hor}}, w^{\text{hor}})
= \int_{\mathbb R^d}\langle L(v^{\text{hor}}),w^{\text{hor}} \rangle\,dx
\\&
= \int_{\mathbb R^d}\Big\langle \sum_{i=1}^n\de(x-q_i)(K\i(q)v)_i,
 \sum_{j=1}^n k(x-q_j)(K\i(q)w)_j \Big\rangle\,dx
 \\
&= \sum_{i,j=1}^n \langle v_i, (K(q)^{-1})_{ij} w_j\rangle = v^TK\i(q)w.
\end{align*}
For a conditionally positive kernel as described in \ref{sec:cpd_kernels}, the evaluation $\sum_{i=1}^n p_i\otimes\de_{q_i}$ is not in general continuous on the completion of $\X_{\mathcal{A}}(\mathbb R^d)$ under the kernel norm. The functional is only continuous if $q_i,p_i$ satisfy the constraints \eqref{eq:cpd_constraints}. This concretely manifests itself in the kernel matrix $K(q)$ failing to be positive definite on the full cotangent space, and the above construction of the Riemannian metric therefore breaks down. However, the idea pursued in this paper is that the same description gives a positive definite Riemannian cometric on the constrained subbundle of the cotangent bundle $T^*\on{Land}^n$, and it therefore induces a Riemannian metric on the quotient shape space $\overline{\on{Land}}^n$.

\subsection{Geodesics on the landmark space}
The geodesic equation for the above metric on $\on{Land}^n$ is most conveniently derived on Hamiltonian form. Consider the cotangent bundle  $T^*{\on{Land}^n}=  
\on{Land}^n\x ((\mathbb R^d)^n)^*\ni (q,\al)$. 
For $\alpha,\beta\in\R^d$ we let again $\langle\alpha,\beta\rangle$ be the standard inner product on $\mathbb R^n$. The cometric on $T^*\on{Land}^n$ is then given by  
\begin{align*}
(g^L)^*_q(\al,\be) &= \sum_{i,j=1}^n \langle\al_i, K(q)_{ij}\be_j\rangle
\end{align*}
and we can define the Hamiltonian energy function
\begin{equation*}
H(q,\al)=\tfrac12 (g^L)^*_q(\al,\al) 
.
\end{equation*}
Its Hamiltonian vector field using $\mathbb R^d$-valued derivatives is
\begin{align*}
H_E(q,\al) &=
  \sum_{i,k=1}^n \Big(k(q_k-q_i)\al_i\frac{\p}{\p q_k} 
+ \langle\al_i.Dk(q_i-q_k)\Big(\frac{\p}{\p \al_k}\Big),\al_k\rangle  \Big). 
\end{align*}
which gives the standard Hamiltonian form of the geodesic equation.
Note that the Hamiltonian and Hamiltonian flow is also defined if the metric is not positive definite as we use in the paper. In such cases, the Hamiltonian may be negative, but, in the conditionally positive definite case, the constraints will ensure positive energy which will be preserved by the flow.

\subsection{The action of $\SE(d)$ on $\Diff_{\mathcal{A}}(\mathbb R^d)$}
Note the following facts for the outer action of $\SE(d)$ on $\Diff_{\mathcal{A}}(\mathbb R^d)$ given in Section~\ref{sec:outer_action}:

\begin{itemize}
\item The identity $\on{Id}_{\mathbb R^d}\in \Diff_{\mathcal{A}}(\mathbb R^d)$ is a fixed point for the action, thus the action is not proper. But $\on{Id}_{\mathbb R^d}$ is the only diffeomorphism with non-compact isotropy group. The restriction of the $\SE(d)$-action to the complement of $\on{Id}_{\mathbb{R}^d}$ in  $\Diff_{\mathcal{A}}(\mathbb R^d)$ is proper: We check \cite[6.20.2]{michorTopicsDifferentialGeometry2008} adapted to this infinite dimensional situation. Suppose that $\ph_n\to \ph$ and $A_n^{-1}\o\ph_n\o A_n \to \ps(x)$ in $\Diff_{\mathcal{A}}(\mathbb R^d)$ for sequences or nets, where $A_n(x)= R_nx +b_n$. Since $\on{SO}(d)$ is compact, by passing to a subsequence, we may assume that $R_n\to R$. Thus $\ph(Rx+b_n)\to R\ps(x)+b_n$. Since $\ps\ne \on{Id}$ there exists $x_0$ with $|R\ps(x_0)-R(x_0)| = C>0$. If a subsequence of $(b_n)$ is unbounded then eventually, since $\ph$ falls to $\on{Id}$ towards infinity, 
$$
\ep>|\ph(Rx_0+b_n) - Rx_0 - b_n|\to |R\ps(x_0) +b_n -Rx_0 -b_n| = C,
$$ 
a contradiction. Thus $(b_n)$ is a bounded sequence and admits a converging subsequence.
\item On a generic diffeomorphism $\ph$, the  action is free, but this is not true in general: Consider $\ph(x) := \frac{r(|x|)}{|x|}x$ where $r:\mathbb R_{\ge0}\to \mathbb R_{\ge0}$ is smooth, $r'(t)\ge\ep>0$ everywhere, and $r(t) =t$ for $t>C$. Then $\ph$ commutes with every element of $\SO(d)$. 
\item The orbit space $\Diff_{\mathcal{A}}(\mathbb R^d)/\SE(d)$ of $\SE(d)$-conjugacy classes of $\Diff_{\mathcal{A}}(\mathbb R^d)$ is Hausdorff. It is stratified into orbit type strata which correspond to conjugacy classes of closed subgroups of $\SE(d)$, i.e., conjugacy classes of stabilizer groups of diffeomorphisms. 
\end{itemize}

\subsection{Geometry of orbit spaces with induced metrics}
\label{sec:orbit_spaces}
We can use general results on the induced metric on orbit spaces to get information about the geometry of $\overline{\on{Land}}^n \cong \on{Land}^n_0/\on{SO}(d)$:

\begin{theorem}{\rm\cite[3.1 -- 3.5]{Michor03orbit}}
Let  $(M,g)$  be a complete  
Riemannian manifold with a proper isometric action of a Lie group $G$ . Then the following holds:
\begin{enumerate}
\item Let $c(t)$ be a geodesic in $M$ and for $X\in \mathfrak{g}$ let $\ze_X\in \X(M)$ be the infinitesimal action vector field. Then $t\mapsto g(c'(t),\ze_X(c(t))$ is constant; see {\rm \cite[30.1]{michorTopicsDifferentialGeometry2008}}. Therefore, if a geodesic is normal to an orbit once, it is normal always.   
\item The orbit space $(M/G, d)$ with the natural metric is a 
      complete metric space and a path metric space. 
\item Any minimal geodesic segment of $M/G$ is the projection of a horizontal (i.e\. orthogonal to orbits) geodesic segment of $M$ which is called a `horizontal lift'.
\item For every $\bar p \in M/G$ there exists $r>0$ such that each $\bar q$ with $d(\bar p,\bar q)<r$ can be connected to $\bar p$ by a unique minimal geodesic segment.
\item Any two horizontal lifts in $M$ of a minimal geodesic segment in $M/G$ differ by the action of an isometry in $G$. For any normal geodesic $c$ in $M$ the projection $\pi\o c$ into $M/G$ has the following property: For each $t$ there exists $r>0$ such that $\pi(c(s))$ for$s$ between $t$ and $t\pm r$ both are minimal geodesic segments.
\item  Let $\bar p \bar q$ be a minimal geodesic  
in $M/G$ of length $d$ and let $pq$ be a horizontal lift. 
Then the stabilizer 
$G_x$ of any interior point of $pq$ is contained in the stabilizers 
$G_p, G_q$ of the end points. 
\item 
$M_{\le(H)}$, the set of orbits with orbit type 
smaller then $(H)$, is a convex subset of 
$M/G$, i.e., any minimal geodesic segment in $M$ between two points in $M_{\le(H)}$ lies in $M_{\le(H)}$.
In particular, $(M/G)_{\operatorname{reg}}$ is a convex  
open dense submanifold. 
\end{enumerate}
\end{theorem} 

While this theorem applies to the positive definite metrics induced by e.g. Sobolev operators, we cannot use it directly for the constructions with the screened elasticity since it doesn't give a Riemannian metric on $\on{Land}^n$. The corollary below rectifies this by allowing the construction of a Riemannian metric on $\on{Land}^n$ from a semi-definite metric, thus allowing the theorem above to be applied. This extension is analogous to extension of CPD metrics to positive definite metrics mentioned in section~\ref{sec:cpd_kernels}. The induced metric on $\overline{\on{Land}}^n$ is not affected by the extension since the quotient removes the rigid-motion component.

\begin{corollary}
Let  $\ell: G\x M\to M$  be a proper smooth action of a compact Lie group with infinitesimal action  $\ze:\mathfrak{g}\to \X(M)$.  
Let $k:T^*M\x_M T^*M \to \mathbb R$ be a smooth symmetric positive semi-definite form such that for each $q\in M$ the null-space $N_q$ of $k_q$ is a linear complement to the annihilator $\ze(\mathfrak{g})(q)^{\on{ann}}$ of the tangent space to the $G$-orbit through $q$, so that $ T^*_qM= N_q\oplus \ze(\mathfrak g)(q)^{\on{ann}}$.

Let $\ga$ be a bi-invariant Riemannian metric on $G$ and for each $G$-orbit $G.q$ let $\ga^{G.q}$ be the induced $G$-invariant Riemannian metric on this orbit, such that $\on{ev_q}: G\to G.q$ is a Riemannian submersion. The duality $\langle \;,\;\rangle:T^*M\x_M TM \to \mathbb R$ restrict to a complete duality between $T_q(G.q)$ and $N_q$ which identifies $N_q$ with $T^*_q(G.q)$. The inverse of $\ga^{G.q}$ thus is a $G$-invariant cometric  $(\ga^{G.q}_q)^{-1}$ on $N_q$. Finally, $q\mapsto (\ga^{G.q}_q)^{-1}+ k_q$ is a $G$-invariant positive definite cometric on $T^*M$ whose inverse is a $G$-invariant Riemannian metric $\tilde K$ on $M$. 
\end{corollary} 
Thus all the results of the theorem above apply now to $M\to M/G$.

\subsection{Non-collision along Hamiltonian trajectories}
\label{app:non-collision}
We here prove Theorem~\ref{thm:non-collision}.
\begin{proof}[Proof of Theorem~\ref{thm:non-collision}]
    Let $r_{ij}(t)=|q_i(t)-q_j(t)|$ denote the Euclidean distance between landmark $i$ and $j$.
    Let $[0,T)$ be a maximal solution interval and suppose for contradiction $T<\infty$. As argued below,  $r_{ij}$ remains bounded both below away from $0$ and above for all $i,j$, and $q(t)$ thus remains in a compact subset of $\on{Land}^n(\R^d)$. 
    We know from Proposition~\ref{prop:momentum_map_preserved} that $\alpha(t)\in\mathfrak{se}(d)(q(t))^{\mathrm{ann}}$. Since $k_s$ is positive definite on $\mathfrak{se}(d)(q(t))^{\mathrm{ann}}$ and the Hamiltonian is constant along the trajectory, $(q(t),\alpha(t))$ stays in a compact subset of $T^*\on{Land}^n(\R^d)$. On this subset, the Hamiltonian vector field is locally Lipschitz and $(q(t),\alpha(t))$ can be extended beyond $T$ by the standard finite dimensional ODE continuation theorem. This gives the contraction on $T<\infty$.

    To show that $r_{ij}$ remains bounded below away from 0 and above, we use that, for a sufficiently small $r_0>0$ and $r_{ij}<r_0$,
    \begin{equation}
        |\dot r_{ij}|
        \leq
        \begin{cases}
        C r_{ij}\sqrt{1+|\log r_{ij}|},
            &s=1+\tfrac d2,\\
        C r_{ij},
            &s>1+\tfrac d2,
        \end{cases}
        \label{eq:small-distance-ode}
    \end{equation}
    and 
    \begin{equation}
        |\dot r_{ij}|
        \leq
        \begin{cases}
        C\sqrt{1+\log(1+r_{ij})},&d=2,\\
        C,&d=3
        \end{cases}
        \label{eq:large-distance-ode}
    \end{equation}
    for some $C>0$.
    These bounds imply that the Osgood integrals 
    $\int_0^{r_0} \frac{dr}{r\sqrt{1+|\log r|}}$, 
    $\int_0^{r_0}\frac{dr}r$, 
    and $\int_1^\infty\frac{dr}{\sqrt{1+\log(1+r)}}$ all diverge, and hence $r_{ij}$ can neither reach zero or infinity in finite time.

    To see that \eqref{eq:small-distance-ode}, \eqref{eq:large-distance-ode} hold, consider the covector $\beta^{ij}\in(\R^d)^n$ with all entries zero except $\beta^{ij}_i=e_{ij}$ and $\beta^{ij}_j=e_{ji}$, $e_{ij}=(q_i-q_j)/r_{ij}$. Note that $\beta^{ij}\in\mathfrak{se}(d)(q))^{\mathrm{ann}}$, and, along the Hamiltonian trajectory, $\dot r_{ij}=K_q(\beta^{ij},\alpha)$. Hence, by Cauchy-Schwarz,
    \begin{equation*}
        |\dot r_{ij}|
        \leq
        \sqrt{K_q(\beta^{ij},\beta^{ij})}
        \sqrt{K_q(\alpha,\alpha)}.
    \end{equation*}
    The rightmost term is constant because the Hamiltonian is constant, and
    \begin{equation*}
        K_q(\beta^{ij},\beta^{ij})
        =
        2e_{ij}^{\mathsf T}
        \bigl(k_s(0)-k_s(r_{ij}e_{ij})\bigr)e_{ij}.
    \end{equation*}
    The bound \eqref{eq:small-distance-ode} then results from the estimate
    \begin{equation*}
        e^{\mathsf T}\bigl(k_s(0)-k_s(re)\bigr)e
        \leq
        \begin{cases}
        Cr^2\bigl(1+|\log r|\bigr),
            &s=1+\tfrac d2,\\
        Cr^2,
            &s>1+\tfrac d2
        \end{cases}
    \end{equation*}
    that can be derived from the Fourier symbol of $k_s$.
    Similarly, \eqref{eq:large-distance-ode} folllows from
    \begin{equation*}
        \|k_s(x)\|
        \leq
        \begin{cases}
        C\bigl(1+\log(1+|x|)\bigr),&d=2,\\
        C,&d=3.
        \end{cases}
    \end{equation*}
    which can be seen from the expression for the Kelvin matrix $\Phi$.

\end{proof}

\subsection{Scale normalization}
\label{app:scale}
Let $N(q)$ be the constrained $K(q)_{ij}^{-1}$-gradient of the radius function $R$ with defining relation for a constrained vector $v$. As in \ref{sec:induced}, let $K(q)_{ij}$ be the kernel block matrix with $d\times d$ block entries. 
We have
\begin{align*}
&dR(q)v = \sum_i\langle C(q)-q_i,C(v)-v_i\rangle  
= \sum_j\big\langle N_j(q), \sum_i K(q)^{-1}_{ji}v_i\big\rangle 
,
\\
&w_j = \sum_i K(q)^{-1}_{ji} v_i \quad \iff \quad
v_i = \sum_j K(q)_{ij} w_j
,
\\
&\sum_j\langle N_j(q), w_j\rangle =
\sum_j\Big\langle \sum_i(C(q)-q_i)\Big(\sum_k K(q)_{kj} - K(q)_{ij}\Big), w_j\Big\rangle
,
\\
&N_j(q) = \sum_i(C(q)-q_i)\Big(\sum_k K(q)_{kj} - K(q)_{ij}\Big)
\\&\qquad\ \,
= \sum_i(C(q)-q_i)\Big(\sum_k k_s(q_k-q_j) - k_s(q_i-q_j)\Big)
\\&\qquad\ \,
= -\sum_i(C(q)-q_i)k_s(q_i-q_j) .
\end{align*}
Note that the vector field $N(q)$ satisfies the contraints  
\[\sum_{i,j}\langle (C(q)-q_i)k_s(q_i-q_j), X(q_j)\rangle = 0\] for all $X\in \mathfrak{se}(d)$ since $X(q_j$ are in the null-space of the kernel $k_s$.

The Poisson bracket on the cotangent space $T^*\on{Land}^n$ with coordinates $(q,\al)$ is given by 
$\{f,g\} = \sum_i(\langle \p_{\al_i}f,\p_{q_i}g\rangle - \langle \p_{q_i}f,\p_{\al_i}g\rangle) $.
We compute the Poisson bracket of the following two functions, where $f(q)$ is a smooth scaling function for the normal field $N(q)$:
\begin{align*}
&F(q,\al) = \sum_j\langle\al_j, f(q).N_j(q)\rangle
= -f(q)\sum_{i,j}k_s(q_i-q_j)\langle \al_j, C(q)-q_i \rangle
,
\\
&\p_{q_\ell}F(q,\al)(p) = - \,\p_{q_\ell}f(q)(p).\sum_{i,j}k_s(q_i-q_j)\langle \al_j, C(q)-q_i\rangle
\\&\qquad
- f(q).\sum_{i,j}\frac1n.k_s(q_i-q_j)\langle \al_j,p\rangle
\\&\qquad
+ f(q).\sum_{j}k_s(q_\ell-q_j)\langle \al_j,p\rangle
\\&\qquad
- f(q).\sum_{j} dk_s(q_\ell-q_j)(p)
\big(\langle \al_j,C(q)-q_\ell\rangle - 
\langle \al_\ell,C(q)-q_j\rangle \big)
,
\\
&\p_{\al_\ell} F(q,\al) = -f(q)\sum_i(C(q)-q_i) k_s(q_i-q_\ell) ,
\\
&H(q,\al) = \frac12\sum_{i,j=1}^n \langle \al_i, k_s(q_i-q_j)\al_j\rangle ,
\\
&\p_{q_\ell}H(q,\al)(p) = \frac12\sum_{i,j=1}^n \langle \al_i, dk_s(q_i-q_j)\big((\de_{\ell i}-\de_{\ell j}).p\big).\al_j\rangle 
\\&\qquad
= \frac12\sum_{j=1}^n \langle \al_\ell, dk_s(q_\ell-q_j)(p).\al_j\rangle
- \frac12\sum_{i=1}^n \langle \al_i, dk_s(q_i-q_\ell)(p).\al_\ell\rangle
\\
&\qquad=  \sum_{j=1}^n \langle \al_\ell, dk_s(q_\ell-q_j)(p).\al_j\rangle ,
\\
&
\p_{\al_\ell}H(q,\al) = \sum_{i,j}\langle \de_{\ell i}, k_s(q_i-q_j)\al_j\rangle = \sum_m k_s(q_\ell - q_m)\al_m .
\end{align*}
We put this together to compute the Poisson bracket.
\begin{align*}
&\{F,H\}(q,\al)= \sum_\ell\Big(\Big\langle \p_{\al_\ell}F(q,\al), 
\p_{q_\ell}H(q,\al)\Big\rangle
- \Big\langle \p_{q_\ell}F(q,\al), \p_{\al_\ell}H(q,\al) \Big\rangle \Big)
\\&
= - f(q)\sum_{\ell, j, i}  k_s(q_i-q_\ell)dk_s(q_\ell-q_j)(C(q)-q_i)\langle \al_\ell, \al_j\rangle
\\&\quad
 + \sum_{\ell,i,j,m}\p_{q_\ell}f(q)(\al_m).k_s(q_\ell - q_m)k_s(q_i-q_j)\langle \al_j, C(q)-q_i\rangle
\\&\quad
+ f(q)\sum_{\ell,i,j,m}\frac1n.k_s(q_i-q_j)k_s(q_\ell - q_m)\langle \al_j,\al_m\rangle
\\&\quad
- f(q)\sum_{\ell,j,m}k_s(q_\ell-q_j)k_s(q_\ell - q_m)\langle \al_j,\al_m\rangle
\\&\quad
+ f(q)\sum_{\ell,j,m} k_s(q_\ell - q_m) dk_s(q_\ell-q_j)(\al_m)
\big(\langle \al_j,C(q)-q_\ell\rangle - 
\langle \al_\ell,C(q)-q_j\rangle \big) .
\end{align*}

\subsection{Affine equivariance of vector-field-valued differential operators}
\label{app:affine-equivariant-operators}

Let \(d\geq 1\), and let
$ \operatorname{Aff}(d)=\operatorname{GL}(d)\ltimes \mathbb{R}^d $
act on \(\mathbb{R}^d\) by 
$g(x)=Ax+b$ for $A\in \operatorname{GL}(d)$ and $b\in \mathbb{R}^d$.
The push forward action on smooth vector fields is 
\[
    (g_*v)(x)
    =
    Dg_{g^{-1}x}\,v(g^{-1}x)
    =
    A\,v(A^{-1}(x-b)).
\]

\begin{proposition}
Any linear differential operator 
$L = C^\infty(\mathbb{R}^d,\mathbb{R}^d)\to
C^\infty(\mathbb{R}^d,\mathbb{R}^d)$
which is equivariant under the affine pushforward action, i.e., $L(g_*v)=g_*(Lv)$ for all \(g\in \operatorname{Aff}(d)\) and every  \(v\in C^\infty(\mathbb{R}^d,\mathbb{R}^d)\), is of order 0 and hence a multiple of the identity.  
\end{proposition}

\begin{proof}
Translation invariant linear differential operators have constant coefficients, which implies that $L$ is of the form
\[
    Lv
    =
    \sum_{|\alpha|\leq \on{deg}(L)} C_\alpha\,\partial^\alpha v \text{ for }C_\al\in \on{End}(\R^d)\,.
\]
Dilation 
$s_r(x)=r x$ for $r>0$ acts on
a vector field by $(s_r)_*v(x) = r v(x/r)$.
For a multi-index \(\alpha\) we get 
$\partial^\alpha\bigl((s_r)_*v\bigr)(x)   = r^{1-|\alpha|}\partial^\alpha v(x/r)$.
Therefore
\begin{align*}    
L((s_r)_*v)(x)
    &= \sum_{|\alpha|\leq m}
    r^{1-|\alpha|}
    C_\alpha\,\partial^\alpha v(x/r).
\\
(s_r)_*(Lv)(x) &= r (Lv)(x/r)
    = \sum_{|\alpha|\leq m} r C_\alpha\,\partial^\alpha v(x/r).
\\&
\implies  \sum_{|\alpha|\leq m}
    \bigl(r^{1-|\alpha|}-r\bigr)
    C_\alpha\,\partial^\alpha v(x/r)
    =
    0\,
\end{align*}
which implies $C_\al =0$ for $|\al|>0$ so that 
$Lv=C_0v$ for a constant matrix \(C_0\in \operatorname{End}(\mathbb{R}^d)\).
Now equivariance under \(\operatorname{GL}(d)\) implies that $C_0$ commutes with any $A\in \on{GL}(\R^d)$ so that finally $C_0=c.\on{Id}$. 
\end{proof}

\begin{remark}
The conclusion depends on the operator being linear and vector-field-valued,
with the pushforward representation used on both sides. It does not rule out
affine-equivariant operators with tensor-valued outputs. For example,
\[
    v\longmapsto D^2v
\]
is affine-equivariant as a map from vector fields to the space of sections of
\[
    T\mathbb{R}^d\otimes S^2T^*\mathbb{R}^d,
\]
since
\[
    D^2(g_*v)(x)[u,w]
    =
    A\,D^2v(g^{-1}x)[A^{-1}u,A^{-1}w].
\]
This still has constant coefficients.
\end{remark}

\bibliographystyle{alpha}
\bibliography{library}

\end{document}